\documentclass{article}

\usepackage[preprint]{neurips_2020}

\usepackage{stmaryrd}
\usepackage{natbib}
\usepackage[utf8]{inputenc} 
\DeclareUnicodeCharacter{FB01}{fi}
\usepackage[T1]{fontenc}    
\usepackage{url}            
\usepackage{booktabs}       
\usepackage{amsfonts}       
\usepackage{nicefrac}       
\usepackage{microtype}      
\usepackage{amsmath}
\usepackage{cleveref}
\usepackage{algorithm}
\usepackage{algpseudocode}
\usepackage{graphicx}
\usepackage{multirow}
\usepackage{amssymb,amsthm,amsfonts,amscd,mathrsfs}
\usepackage{thm-restate}

\newtheorem{thm}{Theorem}

\newcommand*{\defeq}{\stackrel{\Delta}{=}}

\title{Scaling Equilibrium Propagation to Deep ConvNets \\
        by Drastically Reducing its Gradient Estimator Bias}

\author{
    {\bfseries Axel Laborieux $^{1, *}$, Maxence Ernoult$^{1,2, *}$,  Benjamin Scellier$^{3, \dagger}$,\vspace{0.2cm}} \\
    {\bfseries Yoshua Bengio$^{3,4}$, Julie Grollier$^2$, Damien Querlioz$^1$ \vspace{0.2cm}}\\
    $^1$Université Paris-Saclay, CNRS, C2N, 91120, Palaiseau, France \\
    $^2$Unité Mixte de Physique, CNRS, Thales, Université Paris-Saclay\\
    $^3$Mila, Université de Montréal\\
    $^4$Canadian Institute for Advanced Research\\
    ${}^\dagger$ Currently at Google\\
    \textsuperscript{*} Corresponding authors: \texttt{\{axel.laborieux, maxence.ernoult\}@c2n.upsaclay.fr}
}

\begin{document}

\maketitle

\begin{abstract}
 Equilibrium Propagation (EP) is a biologically-inspired counterpart of Backpropagation Through Time (BPTT) which, owing to its strong theoretical guarantees and the locality in space of its learning rule, fosters the design of energy-efficient hardware dedicated to learning.
 In practice, however, EP does not scale to visual tasks harder than MNIST.
 In this work, we show that a bias in the gradient estimate of EP, inherent in the use of finite nudging, is responsible for this phenomenon and that cancelling it allows training deep ConvNets by EP, including architectures with distinct forward and backward connections.
 These results highlight EP as a scalable approach to compute error gradients in deep neural networks, thereby motivating its hardware implementation. 
\end{abstract}

\subsection*{Introduction}
\label{sec:intro}
\noindent How synapses in hierarchical neural circuits are adjusted throughout learning a task remains a challenging question called the credit assignment problem \citep{richards2019deep}. 
Equilibrium Propagation (EP) \citep{scellier2017equilibrium} provides a biologically plausible solution to this problem in artificial neural networks.
EP is an algorithm for convergent RNNs which, by definition, are given a static input and whose recurrent dynamics converge to a steady state corresponding to the prediction of the network. 
EP proceeds in two phases which both involve the same dynamics. First, the network relaxes to a steady state, then the output layer is nudged towards a ground-truth target until a second steady state is reached. 
During the second phase, the perturbation at the output propagates to upstream layers, creating local error signals that match exactly those computed by Backpropagation Through Time (BPTT) \citep{ernoult2019updates}.
The spatial locality of the learning rule prescribed by EP is highly attractive for designing energy-efficient ``neuromorphic'' hardware implementations of gradient-based learning algorithms.
Deep learning models consume orders of magnitude more energy than the brain to learn cognitive  owing to the physical separation of memory and computation on conventional hardware
\citep{strubell2019energy, markovic2020physics}. 
On (non Von Neumann) architectures where the memory is brought at the location of  computation, the implementation of EP could potentially be more energy efficient than the one of Backpropagation on GPUs by at least two order of magnitudes \citep{ambrogio2018equivalent}.


However, previous works on EP \citep{scellier2017equilibrium,o2018initialized,o2019training,ernoult2019updates,ernoult2020equilibrium} limited their experiments to the MNIST classification task and to shallow network architectures.
Despite the theoretical guarantees of EP, the literature suggests that no implementation of EP has thus far succeeded to match the performance of standard deep learning approaches to train deep networks on challenging visual tasks.
In this work, we show that performing the gradient computation phase of EP with nudging strength of constant sign induces a systematic first order bias in the EP gradient estimate which, once cancelled, unlocks the training of deep ConvNets. 
By introducing a new method to estimate the gradient of the loss based on three steady states instead of two, we achieve 11.68\% test error on CIFAR-10, with only 0.6 \% accuracy degradation with respect to BPTT. Standard EP with two steady states yields 86.64\% test error.
We also propose to implement the neural network predictor as an external softmax readout, subsequently allowing us to use the cross-entropy loss, contrary to previous approaches using the squared error loss. Finally, based on ideas of \citet{scellier2018generalization} and \citet{kolen1994backpropagation}, we adapt the learning rule of EP for architectures with distinct forward and backward connections, yielding only 1.5\% accuracy degradation on CIFAR-10 compared to bidirectional connections.

\begin{figure}[ht!]
  \centering
  \includegraphics[width=0.91\textwidth]{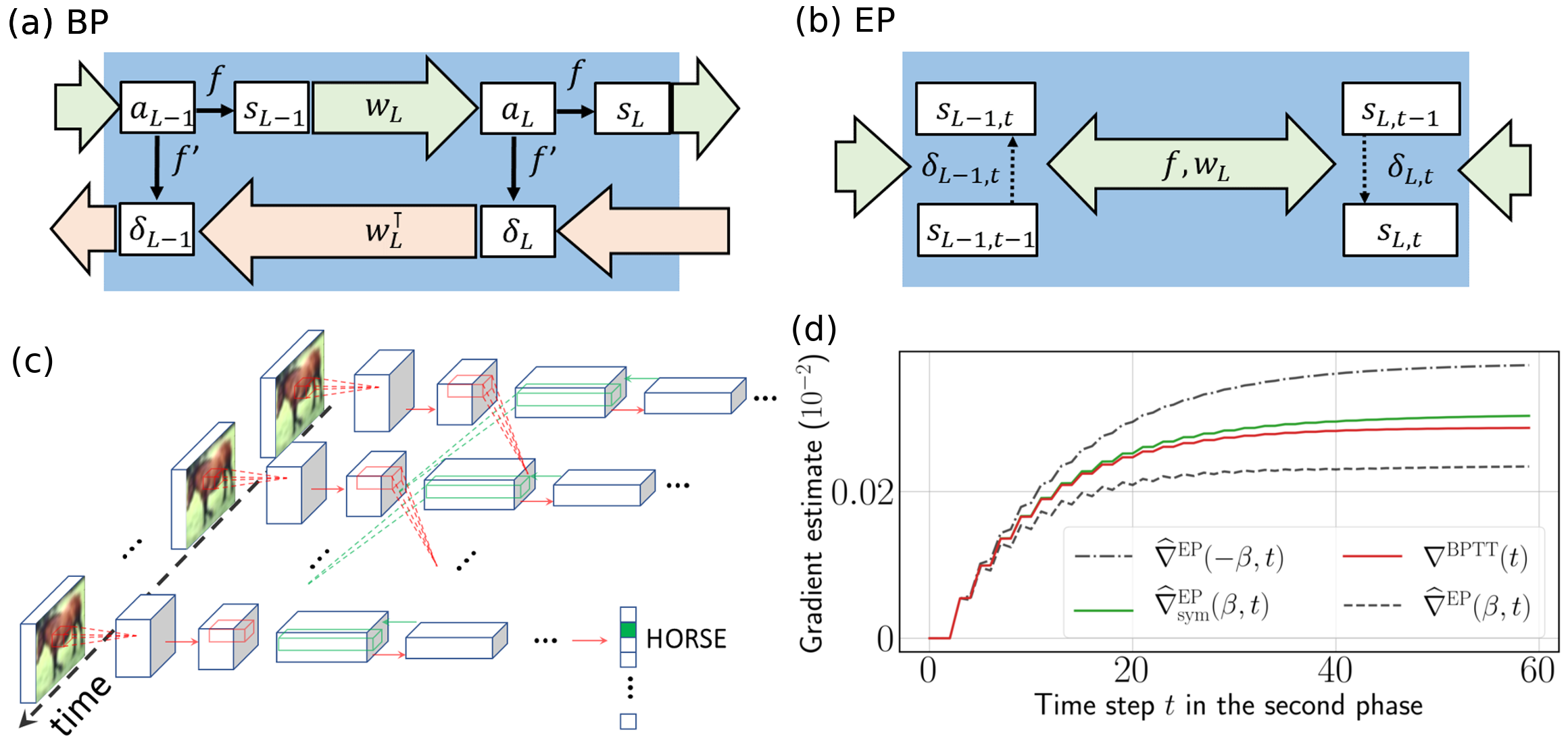}
  \caption{(a), (b) Comparison between BP and EP in terms of computation. $a_L$, $s_L$, $\delta_L$ denote respectively the pre-activation, the activation and the error. (c) Schematic of the recurrent convolutional architecture. (d) Comparison of the former ($\hat{\nabla}^{\rm EP}$) and of the new ($\hat{\nabla}^{\rm EP}_{\rm sym}$) EP estimates with the BPTT gradient ($\hat{\nabla}^{\rm BPTT}$) as a function of time during the second phase (more instances in Appendix~\ref{sec:appCompEstimate}).}
  \label{fig:architecture}
\end{figure}

\subsection*{Background}
{\bfseries Learning in hardware.} 
The challenge of implementing BP on a neuromorphic system for training is due to its distinct modes of operation during forward and backward passes (Fig.~\ref{fig:architecture}(a)).
During the forward pass, neural activations as well as their point derivatives 
have to be computed and stored  for the backward pass \citep{fumarola2016accelerating}.
Error signals computed with BP during training 
also require high precision computation, which does not allow taking advantage of neuromorphic hardware optimally
\citep{narayanan2017toward}. 
Conversely, EP (Fig.~\ref{fig:architecture}(b)) uses the same mode of operation for the two phases: the neural activations evolve along recurrent dynamics through a unique bidirectional pathway.
The error signal in the second phase is thus implicitly propagated to upstream layers through the relaxation of neural activations. 
Owing to these compelling assets, hardware designers are proposing novel circuits for implementing energy-efficient on-chip learning with EP \citep{zoppo2020equilibrium, kendall2020training, ji2020towards, foroushani2020analog}.

{\bfseries Equilibrium Propagation (EP) \citep{scellier2017equilibrium}.} 
We consider the setting of supervised learning where $x$ is the static input and $y$ is the target. 
The goal of learning is to optimize the network parameters $\theta$ to minimize the loss $\mathcal{L}^{*} = \ell(s_*, y)$ 
(Appendix~\ref{sec:appCode}).
Denoting $s_t$ the state of the neurons at time step $t$, during the first (``free'') phase, the RNN evolves according to the following dynamics:
\begin{equation}
\label{eq:dynamics}
    s_{t+1} = \frac{\partial \Phi}{\partial s}(x, s_{t}, \theta) ,
\end{equation}
where $\Phi$ denotes the scalar primitive of the dynamics, and which converges to a steady-state $s_*$.
These discrete-time equations, introduced in \citep{ernoult2019updates}, are a simplification of the original version of EP for real-time dynamical systems \citep{scellier2017equilibrium}, where the dynamics derive from an ``energy function'' $E$, which plays a similar role to the primitive function $\Phi$.
During the second (``nudged'') phase of EP, the system evolves with an additional nudging term 
added to the dynamics, $\beta$ being a small scalar factor:
\begin{equation}
\label{eq:secondphase}
s_{t+1}^{\beta} = \frac{\partial \Phi}{\partial s}(x, s_{t}^{\beta}, \theta) - \beta \frac{\partial \ell}{ \partial s}(s_{t}^{\beta}, y),
\end{equation}
and converges to a steady-state $s_*^\beta$. 
The gradient estimate of EP then reads
\begin{equation}
\label{eq:secondphasebis}
\quad \widehat{\nabla}^{\rm{EP}}(\beta) \defeq
\frac{1}{\beta} \left( \frac{\partial \Phi}{\partial \theta}(x, s_{*}^{\beta}, \theta) - \frac{\partial \Phi}{\partial \theta}(x, s_{*}, \theta) \right).
\end{equation}
A theorem derived by \citet{ernoult2019updates}, inspired from \citet{scellier2019equivalence}, shows that, provided convergence in the first phase has been reached, these gradient estimates given by EP truncated at any time step $t$ of the second phase match, in the limit $\beta \to 0$, those computed by BPTT truncated at time step $T - t$ (Appendix~\ref{sec:BPTT}). 

{\bfseries Convolutional Architectures for Convergent RNNs.} A convolutional architecture for convergent RNNs 
with static input 
(Fig.~\ref{fig:architecture}(c)) was introduced by \citet{ernoult2019updates} and successfully trained with EP on the MNIST dataset, with 
accuracy approaching the one of
BPTT. Denoting $w_n$ a convolutional or fully connected layer, $\star$ and $\mathcal{P}$ convolution and pooling respectively (with $\tilde{w}$, $\mathcal{P}^{-1}$ denoting their respective inverses) and $\sigma$ an activation function, the dynamics read (details in Appendix~\ref{sec:appConvArch}):
\begin{align}
\label{eq:conv-dynamics}
\left\{
\begin{array}{ll}
s^{n}_{t+1} &= \sigma \left( \mathcal{P}\left(w_{n-1}\star s^{n-1}_t\right) + \tilde{w}_{n+1}\star \mathcal{P}^{-1}\left(s^{n+1}_{t}\right)\right), \qquad \mbox{for convolutional layers,} \\
s^n_{t + 1} &= \sigma \left(w_{n-1}\cdot s^{n-1}_{t} + w^{\top}_{n +1}\cdot s^{n+1}_{t} \right), \qquad \mbox{for fully connected layers.}
\end{array} 
\right.
\end{align}
{\bfseries Equilibrium Propagation with unidirectional synaptic connections.} Previous works have proposed a more general formulation of EP, called \emph{Vector Field} (VF), where the dynamics need not derive from a primitive scalar function.
This allows training networks with distinct forward and backward connections \citep{scellier2018generalization,ernoult2020equilibrium} (see Appendix \ref{sec:appConvAsym} for an explicit expression of the VF gradient estimate), further enhancing the biological plausibility of EP.

\subsection*{Contributions of this work : Scaling EP Training}
{\bfseries Reducing bias and variance in the gradient estimate of the loss function.}  \citet{scellier2017equilibrium} motivate their two-phases procedure with $\beta = 0$ and $\beta > 0$ by using $\hat{\nabla}^{\rm EP}$ (Eq.~\ref{eq:secondphasebis}) to compute an estimate of the derivative $\left. \frac{d}{d\beta} \right|_{\beta=0} \frac{\partial \Phi}{\partial \theta}(x, s_{*}^{\beta}, \theta)$ which they showed to be equal to $-\frac{\partial \mathcal{L}_*}{\partial \theta}$. However, the use of $\beta > 0$ induces a systematic first order bias in this estimation, which we propose to eliminate by performing a third phase with $-\beta$ as the nudging factor (proof in Appendix \ref{sec:DL}), keeping the first and second phases unchanged, along with the following symmetric difference estimate:
\begin{equation}
\label{eq:thirdphase}
\widehat{\nabla}^{\rm{EP}}_{\rm sym}(\beta) \defeq \frac{1}{2\beta} \left( \frac{\partial \Phi}{\partial \theta}(x, s_{*}^{\beta}, \theta) - \frac{\partial \Phi}{\partial \theta}(x, s_{*}^{-\beta}, \theta) \right).
\end{equation}
We call $\widehat{\nabla}^{\rm{EP}}(\beta)$ and $\widehat{\nabla}^{\rm{EP}}_{\rm sym}(\beta)$ the one-sided and symmetric EP gradient estimates respectively.
Note that the first-order term of $\widehat{\nabla}^{\rm{EP}}(\beta)$ could also be cancelled out on average, by choosing the sign of $\beta$ at random with even probability, which we call ``Random Sign''.
We  see qualitatively in Fig.~\ref{fig:architecture}(d) that  $\widehat{\nabla}^{\rm{EP}}_{\rm sym}(\beta)$ gives a better estimate of $\hat{\nabla}^{\rm BPTT}$ (see Appendix \ref{sec:BPTT} for a proof) than $\widehat{\nabla}^{\rm{EP}}(\beta)$.

{\bfseries Changing the loss function.} Previous implementations of EP used the squared error loss as, in this setting, the output neurons can be viewed as a part of the system which bidirectionally interacts with neighboring neurons and the nudging term as an elastic force. We propose here an alternative approach, where the output $\hat{y}_t$ is no longer a subset of neurons of the system but is instead implemented as an external read-out, consisting of the composition of a read-out weight matrix $w_{\rm out}$ 
with the softmax function: $\widehat{y}_t = \mbox{softmax}(w_{\rm out}\cdot s_t)$. This approach enables the use of a cross-entropy cost function (see Appendix~\ref{sec:readout} for implementation details).

{\bfseries Changing the learning rule of vector field EP.}
Until now, training experiments of unidirectional weights EP have performed worse than bidirectional weights EP \citep{ernoult2020equilibrium}. In this work, therefore, we tailor a new learning rule for unidirectional weights (described in  Appendix~\ref{sec:appConvAsym}), where the forward and backward weights undergo the same weight updates, incorporating an equal leakage term. This way, forward and backward weights, although they are independently initialized, naturally converge to identical values throughout the learning process, a technique which is adapted from \citet{kolen1994backpropagation} and helped  improve the performance of Feedback Alignment in Deep ConvNets \citep{akrout2019deep}. We adapt the three-phases procedure detailed above to this setting to compute the common update of forward and backward weights, thereby defining the gradient estimate $\hat{\nabla}^{\rm KP-VF}_{\rm sym}$. We also define $\hat{\nabla}^{\rm VF}_{\rm sym}$ the gradient estimate obtained by applying the three-phases procedure to the former VF approach of \cite{scellier2018generalization,ernoult2020equilibrium}.
\begin{table}[ht!] 
\caption{Accuracy comparison on CIFAR-10 between BPTT and EP with several gradient estimation schemes. Note that the different gradient estimates only apply to EP. We indicate over ﬁve trials the mean and standard deviation in parenthesis for the test error, and the mean train error.}
\label{tab:results}
\centering
\begin{tabular}{llllll}
\cline{1-6} \multirow{2}{*}{Loss Function} & \multicolumn{1}{c}{EP Gradient}
                               & \multicolumn{2}{c}{EP Error (\%)}                                                 & \multicolumn{2}{c}{BPTT Error (\%)}                                  \\ 
\multicolumn{1}{c}{}           & \multicolumn{1}{c}{
Estimate} & \multicolumn{1}{c}{Test}           & \multicolumn{1}{c}{Train} & \multicolumn{1}{c}{Test}                 & \multicolumn{1}{c}{Train} \\ \hline
\multirow{3}{*}{Squared Error} & 2-Phase / $\hat{\nabla}^{\rm EP}$                             & $86.64$ $(5.82)$                   & $84.90$                   & \multirow{3}{*}{$11.10$ $(0.21)$} & \multirow{3}{*}{$3.69$}   \\
                               &  Random Sign                           & $21.55$ $(20.00)$                  & $20.01$                   &                                          &                           \\
                               & 3-Phase / $\hat{\nabla}^{\rm EP}_{\rm sym}$                             & $12.45$  $(0.18)$                  & $7.83$                    &                                          &                           \\ \hline
Cross-Ent.                     & 3-Phase / $\hat{\nabla}^{\rm EP}_{\rm sym}$                           & $\mathbf{11.68}$ $\mathbf{(0.17)}$ & $\mathbf{4.98}$           & $11.12$ $(0.21)$                         & $2.19$                    \\
Cross-Ent. (Dropout)           & 3-Phase / $\hat{\nabla}^{\rm EP}_{\rm sym}$                            & $11.87$ $(0.29)$                   & $6.46$                    & $10.72$ $(0.06)$                         & $2.99$                    \\ \hline
\multirow{2}{*}{Cross-Ent.} & 3-Phase / $\widehat{\nabla}^{\rm{VF}}_{\rm sym}$ & $75.47$ $(4.72)$                   & $78.04$                   & \multirow{2}{*}{$9.46$ $(0.17)$} & \multirow{2}{*}{$0.80$}   \\
                               & 3-Phase / $\widehat{\nabla}^{\rm{KP-VF}}_{\rm sym}$   & $\mathbf{13.15}$ $\mathbf{(0.49)}$ & $8.87$                    &                                  &                           \\ \hline
\end{tabular}
\end{table}

{\bfseries Experimental results.} In Table~\ref{tab:results}, we compare the accuracy achieved by the ConvNet for each EP gradient estimate with the accuracy achieved by BPTT. 
On bidirectional weight architectures, the one-sided gradient estimate  $\hat{\nabla}^{\rm EP}$ leads to unstable training behavior where the network is unable to fit the data. 
When the bias in the gradient estimate is averaged out by choosing at random the sign of $\beta$, the average test error over five runs goes down to $21.55\%$. 
However, one run among the five yielded instability similar to the one-sided estimate, whereas the four remaining runs lead to $12.61\%$ test error and $8.64 \%$ train error. 
This method for estimating the loss gradient thus presents high variance, further experiments in Appendix \ref{sec:appExpDetail} confirm this trend. 
Conversely, the symmetric estimate $\hat{\nabla}^{\rm EP}_{\rm sym}$ enables EP to consistently reach $12.45\%$ test error, with only $1.35\%$ degradation with respect to BPTT. 
Therefore, removing the first-order error term in the gradient estimate is critical for scaling to deeper architectures. 
However, proceeding to this end deterministically (with three phases) rather than stochastically (random sign) seems to be more reliable. 

The new readout scheme introduced to optimize the cross-entropy loss function enables EP to narrow the performance gap with BPTT down to $0.56\%$ while outperforming the Squared Error setting by $0.77\%$. 
Finally, we adapted dropout \citep{srivastava2014dropout} to convergent RNNs (see Appendix~\ref{sec:dropout} for implementation details) to see if the performance could be improved further. However, we can observe in Table~\ref{tab:results} that contrary to BPTT, the EP test error is not improved by adding 
dropout,
which we hypothesize is due to the residual estimation bias of the BPTT gradients by EP.

In the situation when the architecture uses distinct forward and backward weights,
we find that the traditional estimate $\widehat{\nabla}^{\rm{VF}}_{\rm sym}(\beta)$ leads to a poor accuracy with $75.47\%$ test-error and  simultaneously observe that forward and backward weights do not align well. 
Conversely, when using our new estimate $\widehat{\nabla}^{\rm{KP-VF}}_{\rm sym}(\beta)$, a good accuracy is recovered with $1.5 \%$  degradation with respect to the architecture with bidirectional connections, and a $3\%$ degradation with respect to BPTT. 
In this case, forward and backward weights are aligned by epoch 50, as observed in 
Appendix~\ref{sec:appAngle}. 
These results suggest that enhancing forward and backward weights alignment also helps EP training in deep ConvNets. 

{\bfseries Discussion.} In comparison with conventional implementations of EP, our results unveil the necessity to compute better gradient estimates in order to scale EP to deep ConvNets on hard visual tasks. 
Keeping the first order bias in the gradient estimate of EP, as done traditionally, severely impedes the training of these architectures and, conversely, removing it brings EP accuracy on CIFAR-10 close to the one achieved by BPTT.
Employing a new training technique that preserves the spatial locality of EP computations,  our results extend to architectures with distinct forward and backward synaptic connections. 
We  only observe a $1.5\%$ performance degradation with respect to the bidirectional architecture. 
Our three steady states-based gradient estimate comes at a computational cost since one more phase is needed with regards to the conventional implementation. 
In the longer run, the full potential of EP will be best envisioned on neuromorphic hardware, which can sustain fast analog device physics and use them to implement the dynamics of EP intrinsically \citep{romera2018vowel, ambrogio2018equivalent}.
Our prescription to run two nudging phases with opposite nudging strengths could be naturally implemented on such systems, which often function differentially to cancel device inherent biases \citep{bocquet2018memory}.

\section*{Broader Impact}
This work may have a long-term impact on the design of energy-efficient hardware leveraging the physics of the device to perform learning. The demonstration that EP can scale to deep networks may also provide insights to neuroscientists to understand the mechanisms of credit assignment in the brain.
Due to the long term nature of this impact, the positive and negative outcomes of this work cannot yet be stated.

\bibliographystyle{abbrvnat}

\newpage
\appendix
\part*{Appendix}

\section{Pseudo code}
\label{sec:appCode}

\subsection{Random one-sided estimation of the loss gradient}
\label{sec:appendOneSided}

In this appendix, we define the random one-sided estimation used in this work and by \citet{scellier2017equilibrium,ernoult2020equilibrium}.

\begin{algorithm}[H]{\emph{Input}: $x$, $y$, $\theta$, $\eta$.  \\
\emph{Output}: $\theta$.}
    \caption{EP with random one-sided estimation of the loss gradient. We omit the activation function $\sigma$ for clarity.}\label{alg:rndsign}
    \begin{algorithmic}[1]
        \State $s_0 \gets 0$
        \For{$t=0$ to $T$} \Comment{First phase.}
        \State $s_{t+1} \gets \frac{\partial \Phi}{\partial s} (x, s_t, \theta)$
        \EndFor
        \State $s_* \gets s_T$
        \State $\beta \gets \beta \times  {\rm{Bernoulli}(1, -1)}$ \Comment{Random sign.}
        \State $s_{0}^{\beta} \gets s_*$
        \For{$t=0$ to $K$} \Comment{Second phase.}
        \State $s_{t+1}^{\beta} \gets \frac{\partial \Phi}{\partial s} (x, s_t^\beta, \theta) - \beta \frac{\partial \ell}{\partial s}(s_{t}^{\beta}, y)$
        \EndFor
        \State $s_{*}^{\beta} \gets s_{K}^{\beta}$
        \State $\nabla_{\theta}^{\rm{EP}} \gets \frac{1}{\beta} \left( \frac{\partial \Phi}{\partial \theta}(s_{*}^{\beta}, \theta) - \frac{\partial \Phi}{\partial \theta}(s_{*}, \theta) \right)$
        \State $\theta \gets \theta + \eta \nabla_{\theta}^{\rm{EP}}$ \\
        \Return $\theta$
    \end{algorithmic}
\end{algorithm}

\subsection{Symmetric difference estimation of the loss gradient}
\label{sec:appendthirdphase}

In this appendix, we define the estimation procedure using a symmetric difference estimate introduced in this work.

\begin{algorithm}[H]{\emph{Input}: $x$, $y$, $\theta$, $\eta$.  \\
\emph{Output}: $\theta$.}
    \caption{EP with symmetric difference estimation of the loss gradient. We omit the activation function $\sigma$ for clarity.}\label{alg:thirdphase}
        \begin{algorithmic}[1]
        \State $s_0 \gets 0$
        \For{$t=0$ to $T$} 
        \State $s_{t+1} \gets \frac{\partial \Phi}{\partial s} (x, s_t, \theta)$ \Comment{First phase.}
        \EndFor
        \State $s_* \gets s_T$ \Comment{Store the free steady state.}
        \State $s_{0}^{\beta} \gets s_*$ 
        \For{$t=0$ to $K$} 
        \State $s_{t+1}^{\beta} \gets \frac{\partial \Phi}{\partial s} (x, s_{t}^{\beta}, \theta) - \beta \frac{\partial \ell}{\partial s}(s_{t}^{\beta}, y)$ \Comment{Second phase.}
        \EndFor
        \State $s_{*}^{\beta} \gets s_{K}^{\beta}$
        \State $s_{0}^{-\beta} \gets s_*$ \Comment{Back to the free steady state.}
        \For{$t=0$ to $K$}
        \State $s_{t+1}^{-\beta} \gets \frac{\partial \Phi}{\partial s} (x, s_{t}^{-\beta}, \theta) + \beta \frac{\partial \ell}{\partial s}(s_{t}^{-\beta}, y)$ \Comment{Third phase.}
        \EndFor
        \State $s_{*}^{-\beta} \gets s_{K}^{-\beta}$
        \State $\widehat{\nabla}_{\theta}^{\rm{EP}} \gets \frac{1}{2\beta} \left( \frac{\partial \Phi}{\partial \theta}(s_{*}^{\beta}, \theta) - \frac{\partial \Phi}{\partial \theta}(s_{*}^{-\beta}, \theta) \right)$
        \State $\theta \gets \theta + \eta \widehat{\nabla}_{\theta}^{\rm{EP}}$ \\
        \Return $\theta$
    \end{algorithmic}
\end{algorithm}

\section{Training RNNs with BPTT}
\label{sec:BPTT}

The convergent RNNs considered by EP can also be trained by Backpropagation Through Time (BPTT). In this context, BPTT consists in performing the first phase for $T$ time steps until the network reaches the steady state $s_T = s_*$, computing the loss at the final time step and subsequently backpropagating the gradients through the computational graph of the first phase. 

We write $\nabla^{\rm{BPTT}}(t)$ the gradient computed by BPTT truncated to the last $t$ time steps ($T-t, \ldots, T$). To derive it in function of loss gradients, let us rewrite Eq.~(\ref{eq:dynamics}) as $s_{t+1} = \frac{\partial \Phi}{\partial s}(x, s_{t}, \theta_t = \theta)$, where $\theta_t$ denotes the parameter at time step $t$, the value $\theta$ being shared across all time steps. We consider the loss after $T$ time steps $\mathcal{L} = \ell(s_T,y)$. Rewriting the dynamics in such a way enables us to define $\frac{\partial \mathcal{L}}{\partial \theta_t}$ as the sensitivity of the loss with respect to $\theta_t$, when $\theta_0, \ldots, \theta_{t-1}, \theta_{t+1}, \ldots, \theta_{T-1}$ remain fixed (set to the value $\theta$). With these notations, the gradient computed by BPTT truncated to the last $t$ time steps is
\begin{equation}
    \nabla^{\rm{BPTT}}(t) = \frac{\partial \mathcal{L}}{\partial \theta_{T-t}} + \ldots + \frac{\partial \mathcal{L}}{\partial \theta_{T-1}}.
\end{equation}

\section{Error terms in the estimates of the loss gradient}
\label{sec:DL}

In this appendix, we prove Lemma \ref{lma:lemma} which shows that $\widehat{\nabla}^{\rm{EP}}_{\rm sym}(\beta)$ is a better estimate of $- \frac{\partial \mathcal{L}^{*}}{\partial \theta}$ than $\widehat{\nabla}^{\rm{EP}}(\beta)$. First, we recall the theorem proved in \citet{scellier2017equilibrium}.
\begin{thm}[\citet{scellier2017equilibrium}]
\begin{equation}
\left. \frac{d}{d\beta} \right|_{\beta=0} \frac{\partial \Phi}{\partial \theta}(x, s_{*}^{\beta}, \theta) = - \frac{\partial \mathcal{L}^{*}}{\partial \theta}.
\end{equation}
\label{thm:main}
\end{thm}

We also recall that the two estimates (one-sided and symmetric) are, by definition:
\begin{align*}
\widehat{\nabla}^{\rm{EP}}(\beta) & \defeq \frac{1}{\beta} \left( \frac{\partial \Phi}{\partial \theta}(x, s_{*}^{\beta}, \theta) - \frac{\partial \Phi}{\partial \theta}(x, s_{*}, \theta) \right), \\
\widehat{\nabla}^{\rm{EP}}_{\rm sym}(\beta) & \defeq \frac{1}{2\beta} \left( \frac{\partial \Phi}{\partial \theta}(x, s_{*}^{\beta}, \theta) - \frac{\partial \Phi}{\partial \theta}(x, s_{*}^{-\beta}, \theta) \right).
\end{align*}

Finally we recall Lemma \ref{lma:lemma}, for readability.

\begin{restatable}{lma}{estimates}
\label{lma:lemma}
Provided the function $\beta \mapsto \frac{\partial \Phi}{\partial \theta}(x,s_{*}^{\beta}, \theta)$ is three times differentiable, we have, as $\beta \to 0$:
\begin{align*}
\widehat{\nabla}^{\rm{EP}}(\beta) &= - \frac{\partial \mathcal{L}^{*}}{\partial \theta} + \frac{\beta}{2} \left. \frac{d^2}{d\beta^2} \right|_{\beta=0} \frac{\partial \Phi}{\partial \theta}(s_{*}^{\beta}, \theta) + O(\beta^2),\\
\widehat{\nabla}^{\rm{EP}}_{\rm sym}(\beta) &= - \frac{\partial \mathcal{L}^{*}}{\partial \theta} + O(\beta^2).
\end{align*}
\end{restatable}

\begin{proof}[Proof of Lemma \ref{lma:lemma}]
Let us define
\begin{equation*}
    f(\beta) \defeq \frac{\partial \Phi}{\partial \theta}(x,s_{*}^{\beta}, \theta).
\end{equation*}
The formula of Theorem \ref{thm:main} rewrites
\begin{equation*}
    f'(0) = - \frac{\partial \mathcal{L}^{*}}{\partial \theta}.
\end{equation*}
As $\beta \to 0$, we have the Taylor expansion
\begin{equation}
\label{eq:taylor:pos}
f(\beta) = f(0) + \beta f'(0) + \frac{\beta^2}{2} f''(0) +  O(\beta^3).
\end{equation}
With these notations, the one-sided estimate reads
\begin{align*}
\widehat{\nabla}^{\rm{EP}}(\beta) & = \frac{1}{\beta} \left( f(\beta) - f(0) \right) \\
& = f'(0) + \frac{\beta}{2} f''(0) + O(\beta^2) \\
& = - \frac{\partial \mathcal{L}^{*}}{\partial \theta} + \frac{\beta}{2} \left. \frac{d^2}{d\beta^2} \right|_{\beta=0} \frac{\partial \Phi}{\partial \theta}(x, s_{*}^{\beta}, \theta)  + O(\beta^2).
\end{align*}

We can also write a Taylor expansion around $0$ at the point $-\beta$. We have
\begin{equation}
\label{eq:taylor:neg}
f(-\beta) = f(0) - \beta f'(0) + \frac{\beta^2}{2} f''(0) +  O(\beta^3).
\end{equation}
Subtracting Eq.~\ref{eq:taylor:neg} from Eq.~\ref{eq:taylor:pos}, we can rewrite the symmetric difference estimate as
\begin{align*}
\widehat{\nabla}^{\rm{EP}}_{\rm sym}(\beta) & = \frac{1}{2 \beta} \left( f(\beta) - f(-\beta) \right) \\
& = f'(0) +O(\beta^2) \\
& = - \frac{\partial \mathcal{L}^{*}}{\partial \theta} + O(\beta^2).
\end{align*}
The derivative to the third order of $f$ is only used to get the $O(\beta^3)$ term in the expansion Eq.~(\ref{eq:taylor:pos}), it can be changed into $o(\beta^2)$ if we only assume $f$ twice differentiable.
\end{proof}

\section{Changing the loss function}
\label{sec:readout}

We introduce a novel architecture to optimize the cross-entropy loss with EP, narrowing the gap with conventional deep learning architectures for classification tasks. In the next paragraph, we denote $\widehat{y}$ the set of neurons that carries out the prediction of the neural network.

\begin{figure}[ht!]
  \centering
  \includegraphics[width=0.85\textwidth]{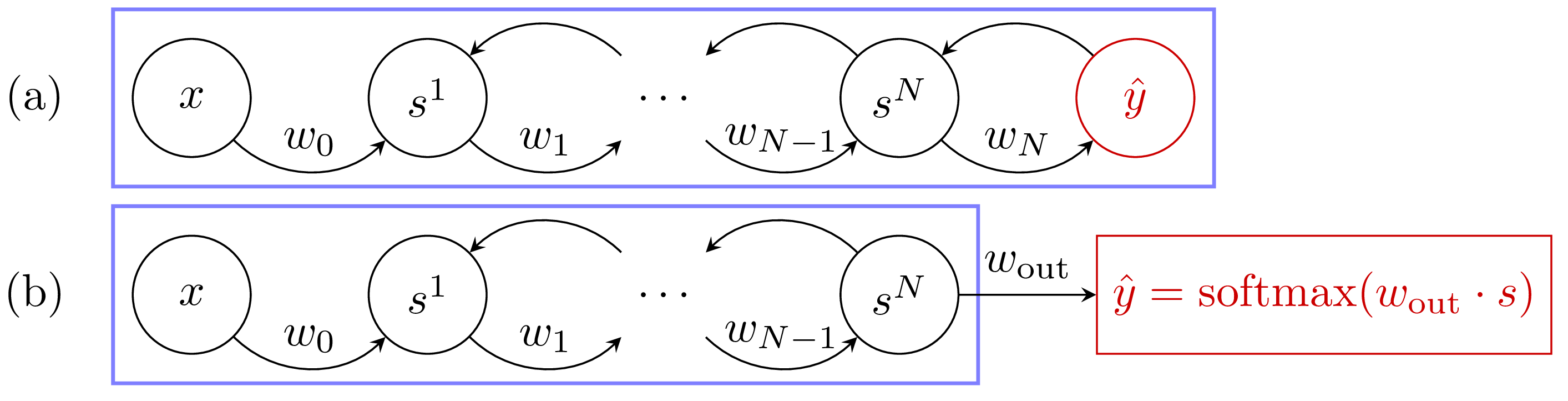}
  \caption{Free dynamics of the architectures used for the two loss functions where the blue frame delimits the system. \textbf{(a) Squared Error loss function.} The usual setting where the predictor $\hat{y}$ (in red) takes part in the free dynamics of the neural network through bidirectional synaptic connections. \textbf{(b)~Cross-Entropy loss function.} The new approach proposed in this work where the predictor $\hat{y}$ (also in red) is no longer involved in the system free dynamics and is implemented as a softmax readout.  
  }
  \label{fig:archi}
\end{figure}

\paragraph{Squared Error loss function.}
Previous implementations of EP used the squared error loss. Using this loss function for EP is natural, as in this setting, the output $\widehat{y}$ is viewed as a part of $s$ (the state variable of the network), which can influence the state of the network through bidirectional synaptic connections (see Fig.~\ref{fig:archi}). The state of the network is of the form $s=(s^1, \dots, s^N,\widehat{y})$ where $h=(s^1, \dots, s^N)$ represent the ``hidden layers'', and the corresponding cost function is
\begin{equation}
    \ell(\widehat{y},y) = \frac{1}{2} \left\| \widehat{y} - y \right\|^2.
\end{equation}

The second phase dynamics of the hidden state and output layer given by Eq.~(\ref{eq:secondphase}) read, in this context:
\begin{equation}
    \label{eq:elastic}
    h_{t+1}^{\beta} = \frac{\partial \Phi}{\partial h}(x, h_t^{\beta}, \widehat{y}_t^{\beta}, \theta), \qquad \widehat{y}_{t+1}^{\beta} = \frac{\partial \Phi}{\partial \widehat{y}}(x, h_t^{\beta}, \widehat{y}_t^{\beta}, \theta) + \beta \; (y - \widehat{y}_t^{\beta}).
\end{equation}

\paragraph{Softmax readout, Cross-Entropy loss function.}
In this paper, we propose an alternative approach, where the output $\widehat{y}$ is not a part of the state variable $s$ but is instead implemented as a read-out (see Fig.~\ref{fig:archi}), which is a function of $s$ and of a weight matrix $w_{\rm out}$ of size $\dim(y) \times \dim(s)$. In practice, $w_{\rm out}$ reads out the last convolutional layer. At each time step $t$ we define:
\begin{equation}
    \widehat{y}_t = \mbox{softmax}(w_{\rm out}\cdot s_t).
\end{equation}
The cross-entropy cost function associated with the softmax readout is then:
\begin{equation}
\label{eq:cross-entropy}
\ell(s,y,w_{\rm out}) = - \sum_{c=1}^C y_c \log(\textrm{softmax}_c(w_{\rm out} \cdot s)).
\end{equation}
Using $\frac{\partial \ell}{\partial s}(s,y,w_{\rm out}) = w_{\rm out}^\top \cdot \left( \textrm{softmax}(w_{\rm out} \cdot s) - y \right)$, the second phase dynamics given by Eq.~(\ref{eq:secondphase}) read in this context:
\begin{equation}
s_{t+1}^{\beta} = \frac{\partial \Phi}{\partial s}(x, s_{t}^{\beta}, \theta) + \beta \; w_{\rm out}^\top \cdot \left( y - \widehat{y}_t^\beta \right).
\end{equation}
Note here that the loss $\mathcal{L}^* = \ell(s_*,y,w_{\rm out})$ also depends on the parameter $w_{\rm out}$. Appendix~\ref{sec:appConvSym-CE} provides the learning rule applied to $w_{\rm out}$.

\section{Convolutional RNNs}
\label{sec:appConvArch}

Throughout this section, $N^{\rm conv}$ and $N^{\rm fc}$ denote respectively the number of convolutional layers and fully connected layers in the convolutional RNN, and $N^{\rm tot} \defeq N^{\rm conv} + N^{\rm fc}$. The neuron layers are denoted by $s$ and range from $s^{0}=x$ the input to the output $s^{N^{\rm tot}}$ in the case of squared error, or $s^{N^{\rm tot}-1}$ in the case of softmax read-out. 

\subsection{Definition of the operations}
\label{sec:appConvOP}
In this subsection we detail the operations involved in the dynamics of a convolutional RNN.

\begin{itemize}
    \item The 2-D convolution between $w$ with dimension $(C_{\rm out}, C_{\rm in}, F,F)$ and an input $x$ of dimensions $(C_{\rm in}, H_{\rm in}, W_{\rm in})$ and stride one is a tensor $y$ of size $(C_{\rm out}, H_{\rm out}, W_{\rm out})$ defined by:
    \begin{equation}
        y_{c,h,w} = (w \star x)_{c,h,w} = B_c + \sum_{i=0}^{C_{\rm in}-1} \sum_{j=0}^{F-1}\sum_{k=0}^{F-1} w_{c,i,j,k}x_{i,j+h,k+w}, 
    \end{equation}
    where $B_c$ is a channel-wise bias.
    \item The 2-D transpose convolution of $y$ by $\tilde{w}$ is then defined in this work as the gradient of the 2-D convolution with respect to its input:
    \begin{equation}
        (\tilde{w} \star y) \defeq \frac{\partial (w \star x)}{\partial x}\cdot y
    \end{equation}
    \item The dot product ``$\bullet$'' generalized to pairs of tensors of same shape $(C,H,W)$:
    \begin{equation}
        a \bullet b = \sum_{c=0}^{C-1}\sum_{h=0}^{H-1}\sum_{w=0}^{W-1} a_{c,h,w}b_{c,h,w}.
    \end{equation}
    \item The pooling operation $\mathcal{P}$ with stride $F$ and filter size $F$ of $x$:
    \begin{equation}
        \mathcal{P}_{F}(x)_{c,h,w} = \underset{i,j \in [0,F-1]}{\rm max}  \left\{ x_{c, F(h-1)+1+i, F(w-1)+1+j} \right\},
    \end{equation}
    with relative indices of maximums within each pooling zone given by:
    \begin{equation}
        {\rm ind}_{\mathcal{P}}(x)_{c,h,w} = \underset{i,j \in [0,F-1]}{\rm argmax}  \left\{ x_{c, F(h-1)+1+i, F(w-1)+1+j} \right\} = (i^{*}(x,h), j^{*}(x,w)).
    \end{equation}
    \item The unpooling operation $\mathcal{P}^{-1}$ of $y$ with indices ${\rm ind}_{\mathcal{P}}(x)$ is then defined as:
    \begin{equation}
        \mathcal{P}^{-1}(y, {\rm ind}_{\mathcal{P}}(x))_{c,h,w} = \sum_{i,j} y_{c,i,j}\cdot \delta_{h, F(i-1)+1+i^{*}(x,h)} \cdot \delta_{w, F(j-1)+1+j^{*}(x,w)},
    \end{equation}
    which consists in filling a tensor with the same dimensions as $x$  with the values of $y$ at the indices ${\rm ind}_{\mathcal{P}}(x)$, and zeroes elsewhere. For notational convenience, we omit to write explicitly the dependence on the indices except when appropriate. 
    \item The flattening operation $\mathcal{F}$ is defined as reshaping a tensor of dimensions $(C,H,W)$ to $(1, CHW)$. We denote by $\mathcal{F}^{-1}$ its inverse.
\end{itemize}

\subsection{Definition of the primitive function}
\label{sec:appConvPhi}
For notational simplicity here, whether $w_n$ is a convolutional layer or a fully connected layer is implied by the operator, respectively $\star$ for convolutions and $\cdot$ for linear layers. The primitive function can therefore be defined as:
\begin{equation}
\label{eq:phiCNN}
    \Phi(x, \{s^{n}\}) = \sum_{n =0}^{N^{\rm conv}-1} s^{n+1}\bullet\mathcal{P}\left(w_{n+1}\star s^{n}\right) + \sum_{n = N^{\rm conv}}^{N^{\rm tot}-1} s^{n \top}\cdot w_{n+1}\cdot s^{n+1}, 
\end{equation}
where $\bullet$ is the Euclidean scalar product generalized to pairs of tensors with same arbitrary dimension, and $\mathcal{P}$ is a pooling operation. 

\subsection{Convolutional RNNs with bidirectional connections}
\label{sec:appConvSym}

In this section, we write explicitly the dynamics and the learning rules applied for the convolutional architecture with bidirectional connections, for the Squared loss function and the Cross-Entropy loss function, for the one-sided and symmetric estimates.

\subsubsection{Squared Error loss}
\label{sec:appConvSym-SE}
\paragraph{Equations of the dynamics.} In this case, the dynamics read:

\begin{align}
\label{eq:conv-archi-sym}
\left\{
\begin{array}{l}
\displaystyle s^{n+1}_{t+1} = \sigma \left( \mathcal{P}(w_{n+1} \star s^{n}_{t}) + \tilde{w}_{n+2} \star \mathcal{P}^{-1}(s^{n+2}_{t}) \right), \qquad \forall n \in [0, N^{\rm conv}-2] \\
\displaystyle s^{N^{\rm conv}}_{t+1} = \sigma \left( \mathcal{P}(w_{N^{\rm conv}} \star s^{N^{\rm conv}-1}_{t}) + \mathcal{F}^{-1}({w_{N^{\rm conv}+1}}^{\top} \cdot s^{N^{\rm conv}+1}_{t}) \right), \\
\displaystyle s^{N^{\rm conv} + 1}_{t+1} = \sigma \left( w_{N^{\rm conv} + 1} \cdot \mathcal{F}(s^{N^{\rm conv}}_{t}) + {w_{N^{\rm conv} + 2}}^{\top} \cdot s^{N^{\rm conv} + 2}_{t} \right), \\
\displaystyle s_{t+1}^{n+1} = \sigma \left( w_{n+1} \cdot s^{n}_{t} + {w_{n+2}}^{\top} \cdot s^{n+2}_{t} \right), \qquad \forall n \in [N^{\rm conv} + 1, N^{\rm tot}-2] \\
\displaystyle s_{t+1}^{N^{\rm tot}} = \sigma \left( w_{N^{\rm tot}} \cdot s^{N^{\rm tot}-1}_{t}\right) + \beta(y - s^{N^{\rm tot}}), \quad \text{with $\beta=0$ during the first phase,}
\end{array}
\right.
\end{align}
where we take the convention $s^{0}=x$. In this case, we have $\hat{y} = s_{t+1}^{N^{\rm tot}}$.
Considering the function:

\begin{align*}
    \Phi(x, s^{1}, \cdots, s^{N^{\rm tot}}) &= 
    \sum_{n= N_{\rm conv} + 2}^{N_{\rm tot} - 1} s^{{n + 1}^\top}\cdot w_{n}\cdot s^{n}
    + s^{N_{\rm conv} + 1}\cdot w_{N_{\rm conv} + 1}\cdot \mathcal{F}(s_{t}^{N_{\rm conv}})\\
    &+\sum_{n = 1}^{N_{\rm conv} - 1} s^{n + 1}\bullet\mathcal{P}\left(w_{n + 1}\star s^{n}\right) + s^{1}\bullet\mathcal{P}\left(w_{1}\star x\right), 
\end{align*}

when ignoring the activation function, we have:

\begin{equation}
\label{eq:dphids-SE}
\forall n \in [1, N^{\rm tot}]: \quad s_t^n \approx  \frac{\partial \Phi}{\partial s^n}.
\end{equation}

Note that in the case of the Squared Error loss function, the dynamics of the output layer derive from $\Phi$ as it can be seen by Eq.~(\ref{eq:dphids-SE}). 

\paragraph{Learning rules for the one-sided EP estimator.} In this case, the learning rules read:

\begin{align}
   \left\{
\begin{array}{l}
\forall n \in [N_{\rm conv} + 2, N_{\rm tot} - 1]: \quad \Delta w_{n}  =
\frac{1}{\beta}\left(s_{*}^{n + 1, \beta}\cdot s_{*}^{{n, \beta}^\top} - s_{*}^{n + 1}\cdot s_{*}^{{n}^\top}  \right) \\
\Delta w_{N_{\rm conv}+ 1}  =
\frac{1}{\beta}\left(s_{*}^{N_{\rm conv} + 1, \beta}\cdot \mathcal{F}\left(s^{N_{\rm conv}, \beta}_{*}\right)^\top - s_{*}^{N_{\rm conv} + 1}\cdot \mathcal{F}\left(s^{N_{\rm conv}}_{*}\right)^\top \right) \\
\forall n \in [1, N_{\rm conv} - 1]:\quad \Delta w_{n + 1}  =  \frac{1}{\beta} \left(\mathcal{P}^{-1}(s^{n + 1, \beta}_{*})\star s^{n, \beta}_{*} - \mathcal{P}^{-1}(s^{n + 1}_{*})\star s^{n}_{*} \right)\\
\Delta w_1  =  \frac{1}{\beta} \left(\mathcal{P}^{-1}(s^{1, \beta}_{*})\star x - \mathcal{P}^{-1}(s^{1}_{*})\star x \right)  
\end{array}, 
\right. 
\label{deltaconv-sym-SE}
\end{align}

\paragraph{Learning rules for the symmmetric EP estimator.} In this case, the learning rules read:

\begin{align}
   \left\{
\begin{array}{l}
\forall n \in [N_{\rm conv} + 2, N_{\rm tot} - 1]: \quad \Delta w_{n}  =
\frac{1}{2\beta}\left(s_{*}^{n + 1, \beta}\cdot s_{*}^{{n, \beta}^\top} - s_{*}^{n + 1, -\beta}\cdot s_{*}^{{n, -\beta}^\top}  \right) \\
\Delta w_{N_{\rm conv}+ 1}  =
\frac{1}{2\beta}\left(s_{*}^{N_{\rm conv} + 1, \beta}\cdot \mathcal{F}\left(s^{N_{\rm conv}, \beta}_{*}\right)^\top - s_{*}^{N_{\rm conv} + 1, -\beta}\cdot \mathcal{F}\left(s^{N_{\rm conv}, -\beta}_{*}\right)^\top \right) \\
\forall n \in [1, N_{\rm conv} - 1]:\quad \Delta w_{n + 1}  =  \frac{1}{2\beta} \left(\mathcal{P}^{-1}(s^{n + 1, \beta}_{*})\star s^{n, \beta}_{*} - \mathcal{P}^{-1}(s^{n + 1, -\beta}_{*})\star s^{n, -\beta}_{*} \right)\\
\Delta w_1  =  \frac{1}{2\beta} \left(\mathcal{P}^{-1}(s^{1, \beta}_{*})\star x - \mathcal{P}^{-1}(s^{1, -\beta}_{*})\star x \right)  
\end{array}, 
\right. 
\label{deltaconv-sym-SE2}
\end{align}

\subsubsection{Cross-Entropy loss}
\label{sec:appConvSym-CE}
\paragraph{Equations of the dynamics.} In this case, the dynamics read:

\begin{align}
\label{eq:conv-archi-sym-softmax}
\left\{
\begin{array}{l}
\displaystyle s^{n+1}_{t+1} = \sigma \left( \mathcal{P}(w_{n+1} \star s^{n}_{t}) + \tilde{w}_{n+2} \star \mathcal{P}^{-1}(s^{n+2}_{t}) \right), \qquad \forall n \in [0, N^{\rm conv}-2] \\
\displaystyle s^{N^{\rm conv}}_{t+1} = \sigma \left( \mathcal{P}(w_{N^{\rm conv}} \star s^{N^{\rm conv}-1}_{t}) + \mathcal{F}^{-1}({w_{N^{\rm conv}+1}}^{\top} \cdot s^{N^{\rm conv}+1}_{t}) \right), \\
\displaystyle s^{N^{\rm conv} + 1}_{t+1} = \sigma \left( w_{N^{\rm conv} + 1} \cdot \mathcal{F}(s^{N^{\rm conv}}_{t}) + {w_{N^{\rm conv} + 2}}^{\top} \cdot s^{N^{\rm conv} + 2}_{t} \right), \\
\displaystyle s_{t+1}^{n+1} = \sigma \left( w_{n+1} \cdot s^{n}_{t} + {w_{n+2}}^{\top} \cdot s^{n+2}_{t} \right), \qquad \forall n \in [N^{\rm conv} + 1,N^{\rm tot}-3] \\
\displaystyle s_{t+1}^{N^{\rm tot}-1} = \sigma \left( w_{N^{\rm tot}-1} \cdot s^{N^{\rm tot}-2}_{t}\right) + \beta {w_{\rm out}}^{\top} \cdot (y - \hat{y}) \quad \text{with $\beta=0$ during the first phase,}\\
\displaystyle \hat{y} = {\rm softmax}(w_{\rm out} \cdot s^{N^{\rm tot}-1}_t),
\end{array}
\right.
\end{align}

where we keep again the convention $s^0 = x$. Considering the function:

\begin{align*}
    \Phi(x, s^{1}, \cdots, s^{N^{\rm tot} - 1}) &= 
    \sum_{n= N_{\rm conv} + 1}^{N_{\rm tot} - 2} s^{{n + 1}^\top}\cdot w_{n}\cdot s^{n}
    + s^{N_{\rm conv} + 1}\cdot w_{N_{\rm conv} + 1}\cdot \mathcal{F}(s_{t}^{N_{\rm conv}})\\
    &+\sum_{n = 1}^{N_{\rm conv} - 1} s^{n + 1}\bullet\mathcal{P}\left(w_{n + 1}\star s^{n}\right) + s^{1}\bullet\mathcal{P}\left(w_{1}\star x\right), 
\end{align*}

when ignoring the activation function, we have:

\begin{equation}
\label{eq:dphids-CE}
 \forall n \in [1, N^{\rm tot} - 1]: \quad s_t^n \approx  \frac{\partial \Phi}{\partial s^n}, \qquad  \hat{y} = {\rm softmax}(w_{\rm out} \cdot s^{N^{\rm tot}-1}_t). \\
\end{equation}

Note that in this case and contrary to the Squared Error loss function, the dynamics of the output layer do not derive from the primitive function $\Phi$, as it can be seen from Eq.~(\ref{eq:dphids-CE})

\paragraph{Learning rules for the one-sided EP estimator.} In this case, the learning rules read:

\begin{align}
   \left\{
\begin{array}{l}
\Delta w_{\rm out} = -\left( \widehat{y}_*^\beta - y \right)\cdot s_*^{\beta, N^\top}. \\
\forall n \in [N_{\rm conv} + 2, N_{\rm tot} - 2]: \quad \Delta w_{n}  =
\frac{1}{\beta}\left(s_{*}^{n + 1, \beta}\cdot s_{*}^{{n, \beta}^\top} - s_{*}^{n + 1}\cdot s_{*}^{{n}^\top}  \right) \\
\Delta w_{N_{\rm conv}+ 1}  =
\frac{1}{\beta}\left(s_{*}^{N_{\rm conv} + 1, \beta}\cdot \mathcal{F}\left(s^{N_{\rm conv}, \beta}_{*}\right)^\top - s_{*}^{N_{\rm conv} + 1}\cdot \mathcal{F}\left(s^{N_{\rm conv}}_{*}\right)^\top \right) \\
\forall n \in [1, N_{\rm conv} - 1]:\quad \Delta w_{n + 1}  =  \frac{1}{\beta} \left(\mathcal{P}^{-1}(s^{n + 1, \beta}_{*})\star s^{n, \beta}_{*} - \mathcal{P}^{-1}(s^{n + 1}_{*})\star s^{n}_{*} \right)\\
\Delta w_1  =  \frac{1}{\beta} \left(\mathcal{P}^{-1}(s^{1, \beta}_{*})\star x - \mathcal{P}^{-1}(s^{1}_{*})\star x \right)  
\end{array}.
\right. 
\label{deltaconv}
\end{align}

\paragraph{Learning rules for the symmetric EP estimator.} In this case, the learning rules read:

\begin{align}
   \left\{
\begin{array}{l}
\Delta w_{\rm out} = -\frac{1}{2}\left(\left( \widehat{y}_*^\beta - y \right)\cdot s_*^{\beta, N^\top} + \left( \widehat{y}_*^{-\beta} - y \right)\cdot s_*^{-\beta, N^\top}\right). \\
\forall n \in [N_{\rm conv} + 2, N_{\rm tot} - 2]: \quad \Delta w_{n}  =
\frac{1}{2\beta}\left(s_{*}^{n + 1, \beta}\cdot s_{*}^{{n, \beta}^\top} - s_{*}^{n + 1, -\beta}\cdot s_{*}^{{n, -\beta}^\top}  \right) \\
\Delta w_{N_{\rm conv}+ 1}  =
\frac{1}{2\beta}\left(s_{*}^{N_{\rm conv} + 1, \beta}\cdot \mathcal{F}\left(s^{N_{\rm conv}, \beta}_{*}\right)^\top - s_{*}^{N_{\rm conv} + 1, -\beta}\cdot \mathcal{F}\left(s^{N_{\rm conv, -\beta}}_{*}\right)^\top \right) \\
\forall n \in [1, N_{\rm conv} - 1]:\quad \Delta w_{n + 1}  =  \frac{1}{2\beta} \left(\mathcal{P}^{-1}(s^{n + 1, \beta}_{*})\star s^{n, \beta}_{*} - \mathcal{P}^{-1}(s^{n + 1, -\beta}_{*})\star s^{n, -\beta}_{*} \right)\\
\Delta w_1  =  \frac{1}{2\beta} \left(\mathcal{P}^{-1}(s^{1, \beta}_{*})\star x - \mathcal{P}^{-1}(s^{1, -\beta}_{*})\star x \right)  
\end{array}.
\right. 
\label{deltaconv2}
\end{align}

\subsubsection{Implementation details in PyTorch.} The equation of the dynamics as well as the EP estimates computation can be expressed as derivatives of the primitive function $\Phi$. 
Therefore, the automatic differentiation framework provided by PyTorch can be leveraged to implement implicitly the equations of the dynamics and the EP estimates computation by differentiating $\Phi$. 
Although this implementation is slower than explicitly implementing the equations of the dynamics, it is more flexible in terms of network architecture as $\Phi$ is relatively easy to compute.

\subsection{Convolutional RNNs with unidirectional connections}
\label{sec:appConvAsym}
In this section, we write the explicit definition of the dynamics and the learning rule of a convolutional architecture with unidirectional connections where forward and backward connections are no longer constrained to be equal-valued. In this setting, we use the Cross-Entropy loss function along with a softmax readout to implement the output layer of the network. In this setting, the dynamics of Eq.~(\ref{eq:dynamics}) is changed into the more general form: 
\begin{equation}
\label{eq:general-dynamics}
    s_{t+1} = F(x, s_t, \theta),
\end{equation}
and the original Vector Field learning rule reads:
\begin{equation}
\label{eq:vf-estimate}
    \Delta \theta = \eta \widehat{\nabla}^{\rm{VF}}(\beta), \qquad \text{where} \qquad \widehat{\nabla}^{\rm{VF}}(\beta) \defeq \frac{1}{\beta} \frac{\partial F}{\partial \theta}(x, s_*, \theta)^\top \cdot \left(s^\beta_* - s_*\right),
\end{equation}
where VF stands for Vector Field \citep{scellier2018generalization}. If the transition function $F$ derives from a primitive function $\Phi$ (i.e., if $F = \frac{\partial \Phi}{\partial s}$), then $\widehat{\nabla}^{\rm VF}(\beta)$ is equal to $\widehat{\nabla}^{\rm EP}(\beta)$ in the limit $\beta \to 0$ ( i.e. $\lim_{\beta \to 0}\widehat{\nabla}^{\rm VF}(\beta) = \lim_{\beta \to 0}\widehat{\nabla}^{\rm EP}(\beta)$).

\paragraph{Equations of the dynamics.}
In this setting, the dynamics Eq.~(\ref{eq:conv-archi-sym-softmax}) have simply to be changed into:

\begin{align}
\label{eq:conv-archi-asym-softmax}
\left\{
\begin{array}{l}
\displaystyle s^{n+1}_{t+1} = \sigma \left( \mathcal{P}(w^{\rm f}_{n+1} \star s^{n}_{t}) + \tilde{w}_{n+2}^{\rm b} \star \mathcal{P}^{-1}(s^{n+2}_{t}) \right), \qquad \forall n \in [0, N^{\rm conv}-2] \\
\displaystyle s^{N^{\rm conv}}_{t+1} = \sigma \left( \mathcal{P}(w^{\rm f}_{N^{\rm conv}} \star s^{N^{\rm conv}-1}_{t}) + \mathcal{F}^{-1}({w^{\rm b}_{N^{\rm conv}+1}}^{\top} \cdot s^{1}_{t}) \right), \\
\displaystyle s^{N^{\rm conv} + 1}_{t+1} = \sigma \left( w_{N^{\rm conv} + 1}^{\rm f} \cdot \mathcal{F}(s^{N^{\rm conv}}_{t}) + {w_{N^{\rm conv} + 2}}^{\rm b^{\top}} \cdot s^{N^{\rm conv} + 2}_{t} \right), \\
\displaystyle s_{t+1}^{n+1} = \sigma \left( w^{\rm f}_{n+1} \cdot s^{n}_{t} + {w^{\rm b}_{n+2}}^{\top} \cdot s^{n+2}_{t} \right), \qquad \forall n \in [N^{\rm conv} + 1, N^{\rm tot}-3] \\
\displaystyle s_{t+1}^{N^{\rm tot}-1} = \sigma \left( w^{\rm f}_{N^{\rm tot}-1} \cdot s^{N^{\rm tot}-2}_{t}\right) + \beta {w_{\rm out}}^{\top} \cdot (y - \hat{y}),\\
\displaystyle \hat{y} = {\rm softmax}(w_{\rm out} \cdot s^{N^{\rm tot}-1}_t),
\end{array}
\right.
\end{align}

where we distinguish now between forward and backward connections: $w^{\rm f}_n \neq w^{\rm b}_n \quad \forall n \in [1, N_{\rm tot} - 2]$. 

\paragraph{Original Vector Field learning rule (VF).}

The symmetric version of the original Vector Field learning rule is defined as:

\begin{equation}
\label{eq:sym-vf-estimate}
 \widehat{\nabla}^{\rm{VF}}_{\rm sym}(\beta) \defeq \frac{1}{2\beta} \frac{\partial F}{\partial \theta}(x, s_*, \theta)^\top \cdot \left(s^\beta_* - s^{-\beta}_*\right),
\end{equation}

which yields in the case of softmax read-out:

\begin{align}
   \left\{
\begin{array}{l}
\Delta w_{\rm out} = -\frac{1}{2}\left(\left( \widehat{y}_*^\beta - y \right)\cdot s_*^{\beta, N^\top} + \left( \widehat{y}_*^{-\beta} - y \right)\cdot s_*^{-\beta, N^\top}\right). \\
\forall n \in [N_{\rm conv} + 2, N_{\rm tot} - 2]: \quad \Delta w_{n}^{\rm f}  =
\frac{1}{2\beta} \left(s_{*}^{n + 1, \beta}-s_{*}^{n + 1, -\beta} \right)\cdot s_{*}^{{n}^\top} \\
\forall n \in [N_{\rm conv} + 2, N_{\rm tot} - 2]: \quad \Delta w_{n}^{\rm b}  =
\frac{1}{2\beta} s_{*}^{{n+1}} \cdot \left(s_{*}^{n, \beta}-s_{*}^{n, -\beta} \right)^{\top} \\
\Delta w_{N_{\rm conv}+ 1}^{\rm f}  =
\frac{1}{2\beta}\left(s_{*}^{N_{\rm conv} + 1, \beta} - s_{*}^{N_{\rm conv} + 1, -\beta}\right) \cdot  \mathcal{F}\left(s^{N_{\rm conv}}_{*}\right)^\top \\
\Delta w_{N_{\rm conv}+ 1}^{\rm b}  =
\frac{1}{2\beta} s_{*}^{N_{\rm conv} + 1} \cdot \left(\mathcal{F}\left(s^{N_{\rm conv}, \beta}_{*}\right) -  \mathcal{F}\left(s^{N_{\rm conv}, -\beta}_{*}\right) \right)^{\top}  \\
\forall n \in [1, N_{\rm conv} - 1]:\quad \Delta w_{n + 1}^{\rm f}  =  \frac{1}{2\beta} \left(\mathcal{P}^{-1}(s^{n + 1, \beta}_{*}) - \mathcal{P}^{-1}(s^{n + 1, -\beta}_{*}) \right) \star s^{n}_{*}\\
\forall n \in [1, N_{\rm conv} - 1]:\quad \Delta w_{n + 1}^{\rm b}  =  \frac{1}{2\beta} \mathcal{P}^{-1}(s^{n + 1}_{*})\star \left( s^{n, \beta}_{*} -  s^{n, -\beta}_{*} \right)\\
\Delta w_1  =  \frac{1}{2\beta} \left(\mathcal{P}^{-1}(s^{1, \beta}_{*})\star x - \mathcal{P}^{-1}(s^{1, -\beta}_{*})\star x \right)  
\end{array}.
\right. 
\label{deltaconvAsym}
\end{align}

Importantly, note that $\Delta w^{\rm f}_n \neq \Delta w^{\rm b}_n \quad \forall n \in [1, N_{\rm tot} - 2]$.

\paragraph{Kolen-Pollack algorithm.} When forward and backward weights have a common gradient estimate, and a weight decay term $\lambda$, they converge to the same values. We recall the proof, noting $t$ the iteration step, $\theta_{\rm f}$ and $\theta_{\rm b}$ respectively forward and backward weights, the update rule follows:

\begin{align*}
\left\{
\begin{array}{l}
\displaystyle \theta_{\rm f}(t+1) = \theta_{\rm f}(t) + \Delta \theta_{\rm f} \\
\displaystyle \theta_{\rm b}(t+1) = \theta_{\rm b}(t) + \Delta \theta_{\rm b}
\end{array}
\right..
\end{align*}

We can then write

\begin{align*}
    \theta_{\rm f}(t+1) - \theta_{\rm b}(t+1) &= \theta_{\rm f}(t) - \theta_{\rm b}(t) + \Delta \theta_{\rm f} - \Delta \theta_{\rm b} \\
    &= \theta_{\rm f}(t) - \theta_{\rm b}(t) - \eta \lambda \left( \theta_{\rm f}(t) - \theta_{\rm b}(t) \right) \\
    &= (1 - \eta \lambda) \left( \theta_{\rm f}(t) - \theta_{\rm b}(t) \right),
\end{align*}

where we use the fact that the estimates are the same for both parameters, such that they cancel out. Then by recursion :

\begin{align*}
    \theta_{\rm f}(t) - \theta_{\rm b}(t) &= (1 - \eta \lambda)^{t} \left( \theta_{\rm f}(0) - \theta_{\rm b}(0) \right) \underset{t \to \infty}{\rightarrow} 0, \quad \text{since} \quad |1 - \eta \lambda| < 1.
\end{align*}

\paragraph{Kolen-Pollack Vector Field learning rule (KP-VF).} Assuming general dynamics of the form of Eq.~(\ref{eq:general-dynamics}), we distinguish forward connections $\theta_{\rm f}$ from backward connections $\theta_{\rm b}$ so that $\theta = \{\theta_{\rm f}, \theta_{\rm b}\}$, with $\theta_{\rm f}$ and $\theta_{\rm b}$ having same dimension. Assuming a first phase, a second phase with $\beta > 0$ and a third phase with $-\beta$, we define:
\begin{equation}
\label{eq:KP-VF}
\forall {\rm i \in \{ f,b \} }, \qquad \overline{\nabla_{\theta_{\rm i}}^{\rm VF}}(\beta) = \frac{1}{2 \beta} \left( \frac{\partial F}{\partial \theta_{\rm i}}^\top(x, s_*^\beta, \theta) \cdot s_*^\beta - \frac{\partial F}{\partial \theta_{\rm i}}^\top(x, s_*^{-\beta}, \theta)\cdot s_*^{-\beta} \right)
\end{equation}
and we propose the following update rules:
\begin{align}
\label{eq:new-vf}
\left\{
\begin{array}{l}
\displaystyle \Delta \theta_{\rm f} = \eta \left(\widehat{\nabla}^{\rm{KP-VF}}_{\rm sym}(\beta) - \lambda \theta_{\rm f}\right)\\
\displaystyle \Delta \theta_{\rm b} = \eta \left( \widehat{\nabla}^{\rm{KP-VF}}_{\rm sym}(\beta) - \lambda \theta_{\rm b}\right)
\end{array}
\right.,
\quad \mbox{with} \quad \widehat{\nabla}^{\rm{KP-VF}}_{\rm sym}(\beta) = \frac{1}{2}(\overline{\nabla_{\theta_{\rm f}}^{\rm VF}}(\beta)+\overline{\nabla_{\theta_{\rm b}}^{\rm VF}}(\beta))
\end{align}
where $\eta$ is the learning rate and $\lambda$ a leakage parameter. The estimate $\widehat{\nabla}^{\rm{KP-VF}}_{\rm sym}(\beta)$ can be thought of a generalization of Eq.~(\ref{eq:thirdphase}), as highlighted in Appendix \ref{sec:appConvAsym} with an explicit application of Eq.~(\ref{eq:new-vf}) to a ConvNet. More specifically, applying Eq.~(\ref{eq:KP-VF}) to Eq.~(\ref{eq:conv-archi-asym-softmax}) yields:

\begin{align}
   \left\{
\begin{array}{l}
\forall n \in [N_{\rm conv} + 2, N_{\rm tot} - 2]: \\ \qquad \qquad \qquad  \overline{\nabla_{w_{n}^{\rm f}}^{\rm VF}}(\beta) = \overline{\nabla_{w_{n}^{\rm b}}^{\rm VF}}(\beta) =
\frac{1}{2\beta} \left(s_{*}^{n + 1, \beta}\cdot s_{*}^{{n, \beta}^\top}-s_{*}^{n + 1, -\beta}\cdot s_{*}^{{n, -\beta}^\top} \right) \\
\overline{\nabla_{w_{N_{\rm conv}+ 1}^{\rm f}}^{\rm VF}}(\beta)= \overline{\nabla_{w_{N_{\rm conv}+ 1}^{\rm b}}^{\rm VF}}(\beta) = \\
\qquad \qquad \qquad  \frac{1}{2\beta}\left(s_{*}^{N_{\rm conv} + 1, \beta}\cdot  \mathcal{F}\left(s^{N_{\rm conv}, \beta}_{*}\right)^\top - s_{*}^{N_{\rm conv} + 1, -\beta}\cdot  \mathcal{F}\left(s^{N_{\rm conv}, -\beta}_{*}\right)^\top\right) \\
\forall n \in [1, N_{\rm conv} - 1]: \\ \overline{\nabla_{w_{n + 1}^{\rm f}}^{\rm VF}}(\beta)=  \frac{1}{2\beta} \left(\mathcal{P}^{-1}\left(s^{n + 1, \beta}_{*}, {\rm ind}_{\mathcal{P}}\left(w_{n + 1}^{\rm f}\star s^{n, \beta}_{*}\right)\right)\star s^{n, \beta}_{*}\right. \\
\qquad \qquad \qquad \qquad \qquad \qquad \qquad \left.- \mathcal{P}^{-1}\left(s^{n + 1, -\beta}_{*} , {\rm ind}_{\mathcal{P}}\left(w_{n + 1}^{\rm f}\star s^{n, -\beta}_{*}\right)\right)\star s^{n, -\beta}_{*} \right)\\
\forall n \in [1, N_{\rm conv} - 1]: \\ \overline{\nabla_{w_{n + 1}^{\rm b}}^{\rm VF}}(\beta) =  \frac{1}{2\beta} \left(\mathcal{P}^{-1}\left(s^{n + 1, \beta}_{*} , {\rm ind}_{\mathcal{P}}\left(w_{n + 1}^{\rm b}\star s^{n, \beta}_{*}\right)\right)\star s^{n, \beta}_{*} \right.\\
\qquad \qquad \qquad \qquad \qquad \qquad \qquad \left.  - \mathcal{P}^{-1}\left(s^{n + 1, -\beta}_{*}, {\rm ind}_{\mathcal{P}}\left(w_{n + 1}^{\rm b}\star s^{n, -\beta}_{*}\right)\right)\star s^{n, -\beta}_{*} \right)\\
\end{array}.
\right. 
\label{eq:nabla-KPVF}
\end{align}

Combining Eqs.~(\ref{eq:nabla-KPVF}) with Eq.~(\ref{eq:new-vf}) gives the associated parameter updates. The updates for $w_1$ and $w_{\rm out}$ are the same than those of Eq.~(\ref{deltaconvAsym}).
Importantly, note that while $\forall n \in [N_{\rm conv} + 1, N_{\rm tot} - 2]: \quad  \overline{\nabla_{w_{n}^{\rm f}}^{\rm VF}}(\beta) = \overline{\nabla_{w_{n}^{\rm b}}^{\rm VF}}(\beta)$, we have $\forall n \in [1, N_{\rm conv} - 1]: \quad  \overline{\nabla_{w_{n}^{\rm f}}^{\rm VF}}(\beta) \neq \overline{\nabla_{w_{n}^{\rm b}}^{\rm VF}}(\beta)$ because of inverse pooling. In other words, the updates of the convolutional filters do not solely depend on the pre and post synaptic activations but also on the location of the maximal elements within each pooling window, itself depending on the filter considered. Hence the motivation to average $\overline{\nabla_{w_{n}^{\rm f}}^{\rm VF}}(\beta) $ and $\overline{\nabla_{w_{n}^{\rm b}}^{\rm VF}}(\beta)$ and use this quantity to update to $w_{n}^{\rm b}$ and $w_{n}^{\rm f}$ and apply the Kolen-Pollack technique. 

\paragraph{Implementation details in PyTorch.} The dynamics in the case of unidirectional connections does not derive from a primitive function $\Phi$. Therefore, it is not possible to implicitly get the dynamics by differentiating one primitive function.
A way around is to get the VF dynamics by differentiating one quantity $\tilde{\Phi}^n$ by layer. This quantity is not a primitive function and is especially designed to get the right equations once differentiated. We define $\tilde{\Phi}^{n}(w_{n}^{\rm f}, w_{n+1}^{\rm b}, s^{n-1}, s^{n})$ by:

\begin{align}
   \left\{
\begin{array}{l}
\forall n \in [1, N_{\rm conv} - 1]: \tilde{\Phi}^n = s^n \bullet \mathcal{P}(w_{n}^{\rm f} \star s^{n - 1} ) + s^{n+1} \bullet \mathcal{P}(w_{n+1}^{\rm b}\star s^n) \\
\tilde{\Phi}^{N_{\rm conv}} = s^{N_{\rm conv}} \bullet \mathcal{P}(w_{N_{\rm conv}}^{\rm f} \star s^{N_{\rm conv} - 1} ) + s^{N_{\rm conv}+1} \cdot w_{N_{\rm conv}+1}^{\rm b}\cdot \mathcal{F}(s^{N_{\rm conv}}) \\
\forall n \in [N_{\rm conv}+1, N_{\rm tot} - 1]: \tilde{\Phi}^n = s^n \cdot w_{n}^{\rm f} \cdot s^{n - 1}  + s^{n+1} \cdot w_{n+1}^{\rm b}\cdot s^n \\
\tilde{\Phi}^{N_{\rm tot}-1} = s^{N_{\rm tot}-1} \cdot w_{N_{\rm tot}-1}^{\rm f} \cdot s^{N_{\rm tot} - 2} + \beta \ell( s^{N_{\rm tot}-1}, y, w_{\rm out})
\end{array},
\right.
\label{eq:cheatedPhi}
\end{align}
where $\ell$ is defined by Eq.~(\ref{eq:cross-entropy}), and $\beta=0$ in the first phase. Then, $\forall n \in [1, N^{\rm tot}-1]$, the dynamics of Eq.~(\ref{eq:conv-archi-asym-softmax}) read:

\begin{equation}
    s^{n}_{t+1} = \sigma \left( \frac{\partial \tilde{\Phi}^n}{\partial s^{n}}(w_{n}^{\rm f}, w_{n+1}^{\rm b}, s^{n-1}_{t}, s^{n}_{t})\right).
\end{equation}

The original VF update of Eq.~(\ref{deltaconvAsym}) can be written as $\forall n \in [1, N^{\rm tot}-1], \forall {\rm i} \in \{ {\rm f},{\rm b} \}$:

\begin{equation}
    \Delta w_{n}^{\rm i} = \frac{1}{2\beta}\left( \frac{\partial \tilde{\Phi}^n}{\partial w_{n}^{\rm i}}(s^{n, \beta}_*, s^{n-1}_*) - \frac{\partial \tilde{\Phi}^n}{\partial w_{n}^{\rm i}}(s^{n, -\beta}_*, s^{n-1}_*)\right), 
\end{equation}

and Eq.~(\ref{eq:nabla-KPVF}) as:

\begin{equation}
    \overline{\nabla_{w_{n}^{\rm i}}^{\rm VF}}(\beta) = \frac{1}{2\beta}\left( \frac{\partial \tilde{\Phi}^n}{\partial w_{n}^{\rm i}}(s^{n, \beta}_*, s^{n-1, \beta}_*) - \frac{\partial \tilde{\Phi}^n}{\partial w_{n}^{\rm i}}(s^{n, -\beta}_*, s^{n-1, -\beta}_*)\right). 
\end{equation}

\section{Experimental details}
\label{sec:appExpDetail}

\begin{figure}[ht!]
  \label{fig:mse}
  \centering
  \includegraphics[width=0.8\textwidth]{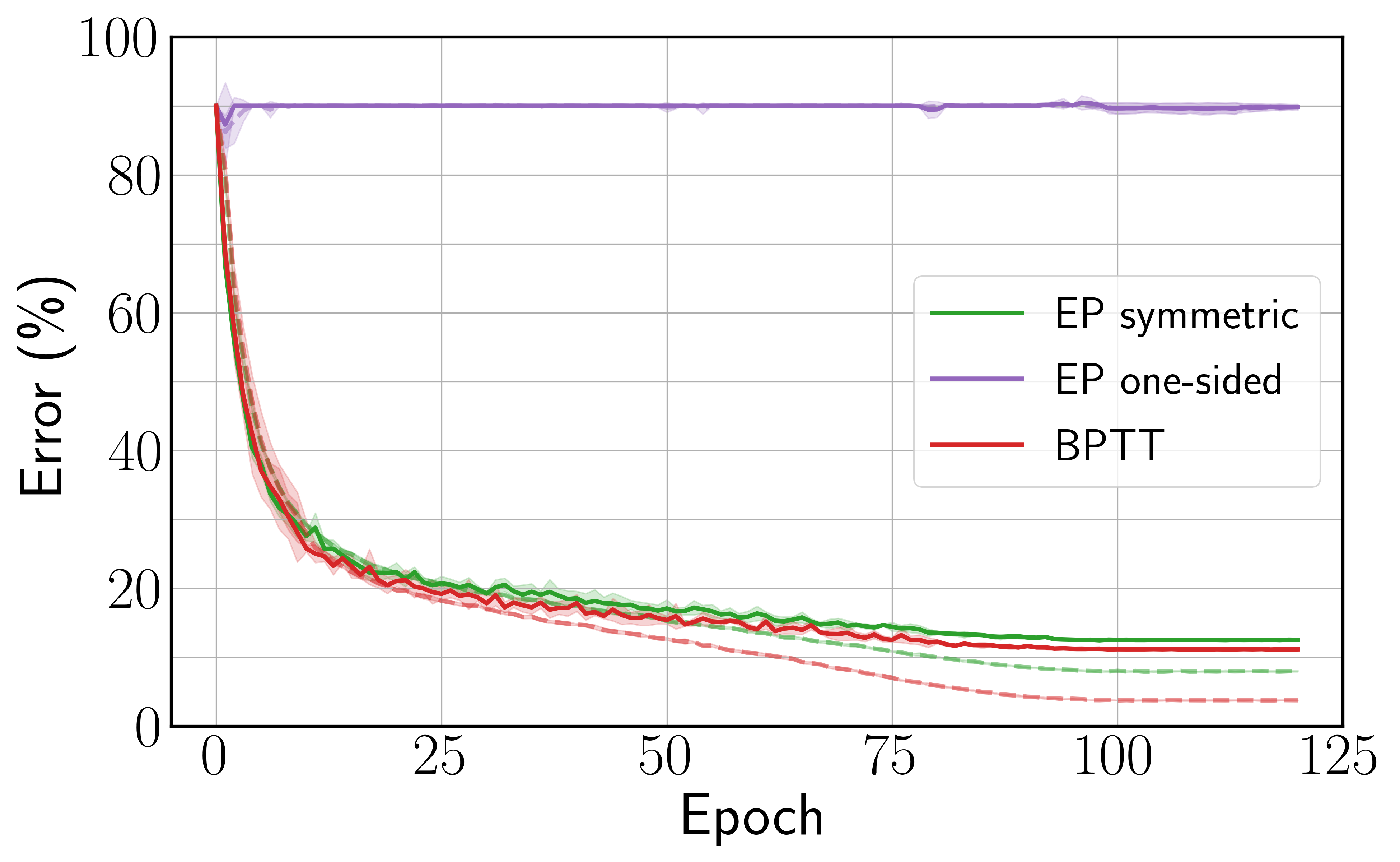}
  \caption{Train (dashed) and test (solid) errors on CIFAR-10 with the Squared Error loss function. The curves are averaged over 5 runs and shadows stand for $\pm 1$ standard deviation.}
\end{figure}

\begin{figure}[ht!]
  \label{fig:cel}
  \centering
  \includegraphics[width=0.8\textwidth]{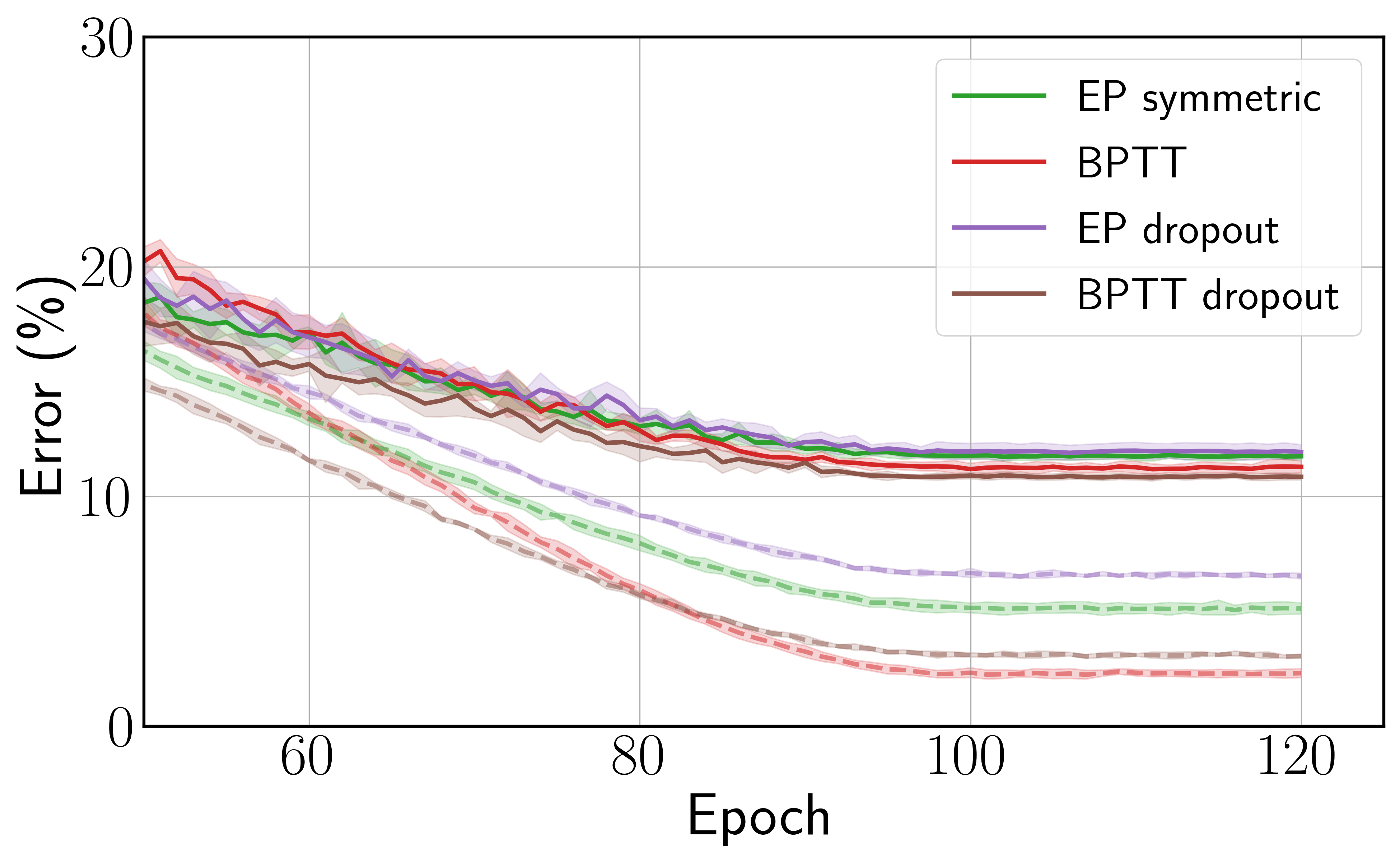}
  \caption{Train (dashed) and test (solid) errors on CIFAR-10 with the Cross-Entropy loss function. The curves are averaged over 5 runs and shadows stand for $\pm 1$ standard deviation.}
\end{figure}

\subsection{Environment and hyper-parameters}
The experiments are run using PyTorch 1.4.0 and torchvision 0.5.0. \citep{paszke2017automatic}. 
The convolutional architecture used in the CIFAR-10 experiment consists of four $3\times3$ convolutional layers of respective feature maps 128 - 256 - 512 - 512. 
We use a stride of one for each convolutional layer, and zero-padding of one for each layer except for the last layer. 
Each layer is followed by a $2\times2$ Max Pooling operation with a stride of two.
The resulting flattened feature vector is of size 512.
The weights are initialized using the default initialization of PyTorch, which is the uniform Kaiming initialization introduced by \citet{he2015delving}.
The data is normalized and augmented with random horizontal flips and random crops.
The training is performed with stochastic gradient descent with momentum and weight decay.
We use the learning rate scheduler introduced by \citet{loshchilov2016sgdr} to speed up convergence.
The simulations were carried across several servers consisting of 14 GPUs in total.
Each run was performed on a single GPU for an average run time of 2 days.

For the unidirectional architecture, the backward weights are defined for all convolutional layers except the first convolutional layer connected to the static input. 
The forward and backward weights are initialized independently at the beginning of training.
The backward weights have no bias contrary to their forward counterparts. 
The hyper-parameters such as learning rate, weight decay and momentum are shared between forward and backward weights.

\begin{table}[ht!]
\caption{Hyper-parameters used for the CIFAR-10 experiments.}
\label{tab:hyperparam}
\centering
\begin{tabular}{ccc}
\hline
Hyper-parameter                                                                  & Squared Error                             & Cross-Entropy                   \\ \hline
$T$                                                                             & 250                             & 250                             \\
$K$                                                                             & 30                              & 25                              \\
$\beta$                                                                         & 0.5                             & 1.0                             \\
Batch Size                                                                      & 128                             & 128                             \\
\begin{tabular}[c]{@{}c@{}}Initial learning rates\\ (Layer-wise)\end{tabular}   & 0.25 - 0.15 - 0.1 - 0.08 - 0.05 & 0.25 - 0.15 - 0.1 - 0.08 - 0.05 \\
Final learning rates                                                            & $10^{-5}$                       & $10^{-5}$                       \\
\begin{tabular}[c]{@{}c@{}}Weight decay\\ (All layers)\end{tabular}             & $3 \cdot 10^{-4}$               & $3 \cdot 10^{-4}$               \\
Momentum                                                                        & 0.9                             & 0.9                             \\
Epoch                                                                           & 120                             & 120                             \\
\begin{tabular}[c]{@{}c@{}}Cosine Annealing \\ Decay time (epochs)\end{tabular} & 100                             & 100                             \\ \hline
\end{tabular}
\end{table}

\subsection{Random-sign estimate variance}

Although not explicitly stated in this purpose, the use of such randomization has been reported in some earlier publications on the MNIST task \citep{scellier2017equilibrium,ernoult2020equilibrium}. However, in this work, we show that this method exhibits high variance in the training procedure.
The results presented in Table \ref{tab:results} consists of five runs. 
In the case of the EP random-sign estimate, one run among the five collapses to random guess similar to the one-sided estimate.
In order to test the frequency of such a phenomenon, we performed another five runs with both symmetric and random-sign estimates. 
The results for each run presented in Table \ref{tab:additionalRuns} show that two trials among ten are unstable, confirming further the high variance nature of the random-sign estimate.

\begin{figure}[ht!]
  \label{fig:rnd-sign}
  \centering
  \includegraphics[width=0.8\textwidth]{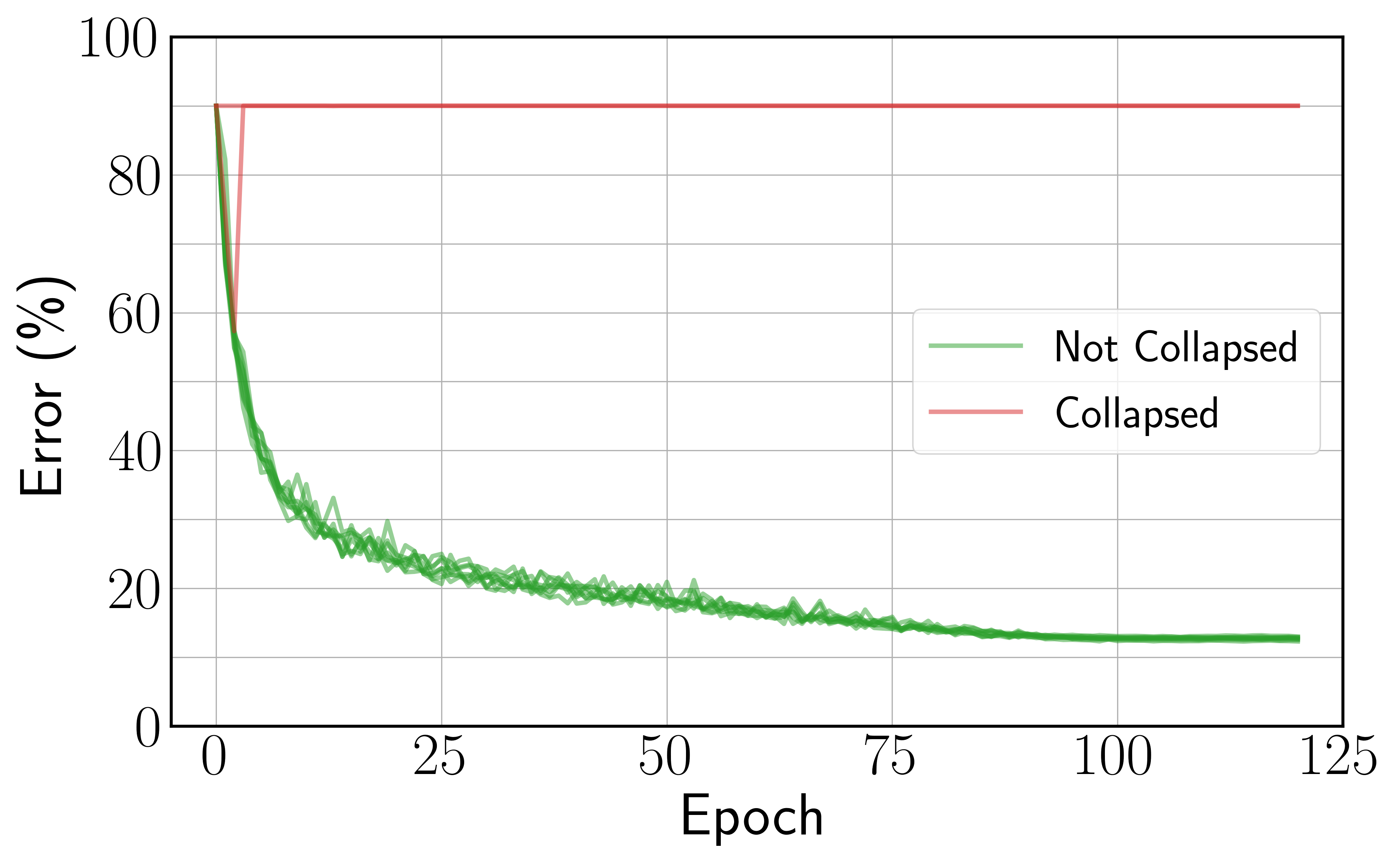}
  \caption{Test error curve of each run with the Squared Error loss function and random-sign estimate. The two collapsed runs among the ten trials are steady to 90\% because in such cases the network typically outputs the same class for each data point.}
\end{figure}

\begin{table}[ht!]
\caption{Best test error comparison between random-sign and symmetric estimates, for ten runs.}
\label{tab:additionalRuns}
\centering
\begin{tabular}{ccc}
\hline
Run index    & EP random-sign   & EP symmetric     \\ \hline
$1$          & $12.97$          & $12.24$          \\
$2$          & $12.72$          & $12.31$          \\
$3$          & $12.30$          & $12.68$          \\
$4$          & $12.45$          & $12.43$          \\
$5$          & $12.78$          & $12.57$          \\
$6$          & $12.66$          & $12.55$          \\
$7$          & $12.84$          & $12.44$          \\
$8$          & $12.59$          & $12.52$          \\
$9$          & $57.32$          & $12.85$          \\
$10$         & $89.98$          & $12.60$          \\ \hline
Mean         & $24.86$          & $\mathbf{12.52}$ \\ \hline
w/o collapse & $\mathbf{12.66}$ & N.A              \\ \hline
\end{tabular}
\end{table}

\subsection{Adding dropout}
\label{sec:dropout}
We adapt dropout \citep{srivastava2014dropout} for convergent RNNs by shutting some units to zero with probability $p<1$ when computing $\Phi(x, s_t, \theta)$. We multiply the remaining active units by the factor $\frac{1}{1-p}$ to keep the same neural activity on average, so that the learning rule is rescaled by $\left(\frac{1}{1-p}\right)^2$. 
The dropped out units are the same within one training iteration but they differ across the examples of one mini batch.
In our experiments we use $p=0.1$ on the last convolutional layer before the linear classifier.
The results are reported in Table~\ref{tab:results}.

\subsection{Changing the activation function}
\label{sec:appActivation}
Previous implementations of EP used a shifted hard sigmoid activation function:
\begin{equation}
    \sigma(x) = \max(0, \min(x, 1)).
    \label{eq:act-former}
\end{equation}
In their experiments with ConvNets on MNIST, \citet{ernoult2019updates} observed saturating units that cannot pass error signals during the second phase. 
In this work, to mitigate this effect, we have rescaled by a factor $1/2$ the slope of the activation function to ease signal propagation and prevent saturation, therefore changing Eq.~(\ref{eq:act-former}) into:
\begin{equation}
    \sigma(x) = \max\left(0, \min\left(\frac{x}{2}, 1\right)\right).
    \label{eq:act-now}
\end{equation}

\section{Weight alignment for unidirectional connections}
\label{sec:appAngle}
The angle $\alpha$ between forward and backward weights is defined as :
\begin{equation}
    \alpha = \frac{180}{\pi}{\rm Acos}\left(\frac{w^{\rm b} \bullet w^{\rm f}}{\|w^{\rm b}\| \|w^{\rm f}\|}\right), \quad \text{where} \quad \|w\| = \sqrt{w \bullet w}.
\end{equation}
Fig.~\ref{fig:angle} shows the angle between forward and backward weights during training on CIFAR-10 for both the original VF learning rule (dashed) and the new learning rule inspired by \citet{kolen1994backpropagation}. 

\begin{figure}[ht!]
  \centering
  \includegraphics[width=0.8\textwidth]{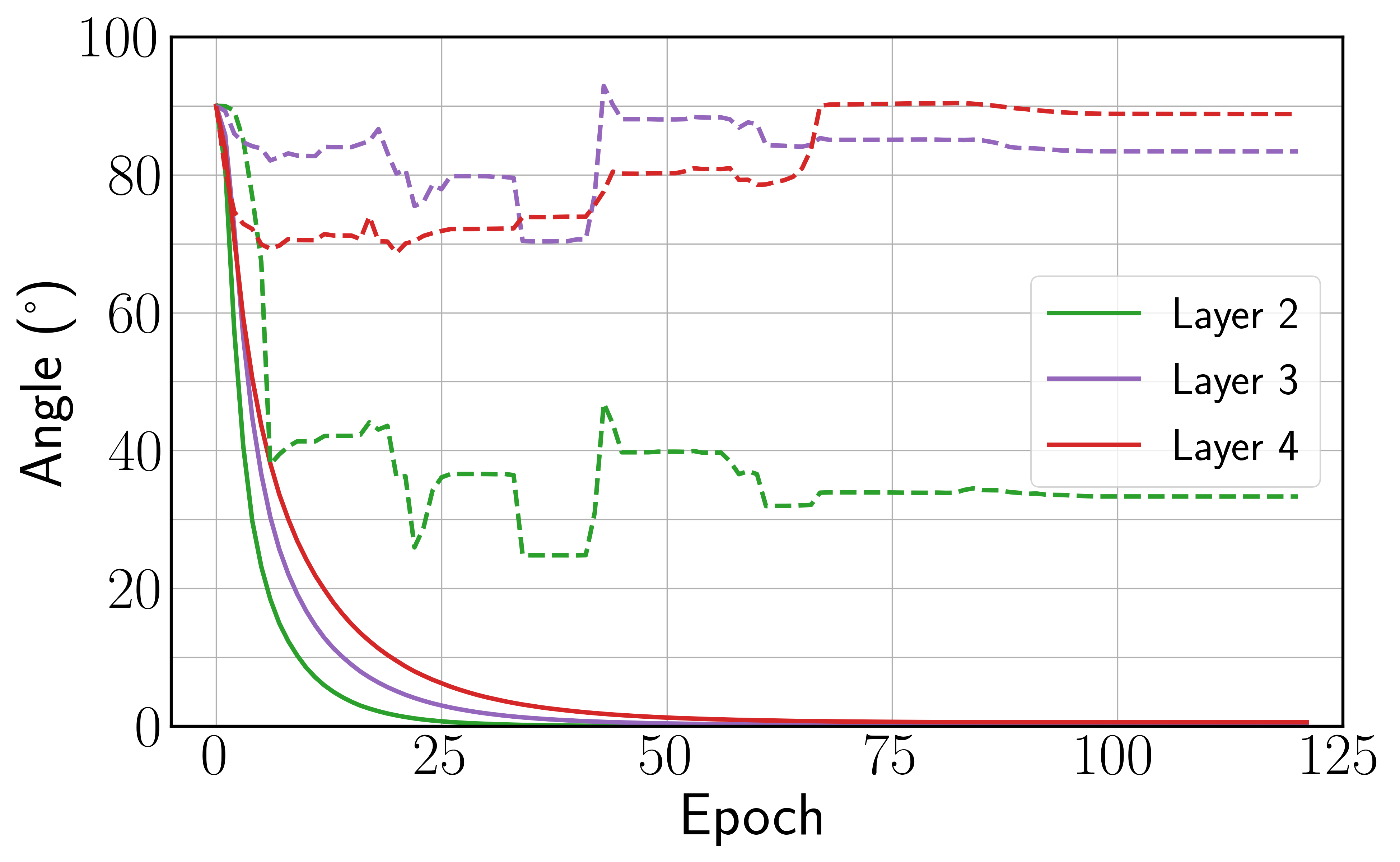}
  \caption{Angle $\alpha$ between forward and backward weights for the new estimate $\widehat{\nabla}^{\rm{KP-VF}}_{\rm sym}$ introduced (solid) and $\widehat{\nabla}^{\rm{VF}}_{\rm sym}$ (dashed).
  }
  \label{fig:angle}
\end{figure}

\section{Layer-wise comparison of EP estimates}
\label{sec:appCompEstimate}

In this section we show on Fig.~\ref{fig:morecurve} more instances of Fig.~\ref{fig:architecture}(d) for each layer of the convolutional architecture.

\begin{figure}[ht!]
  \centering
  \includegraphics[width=\textwidth]{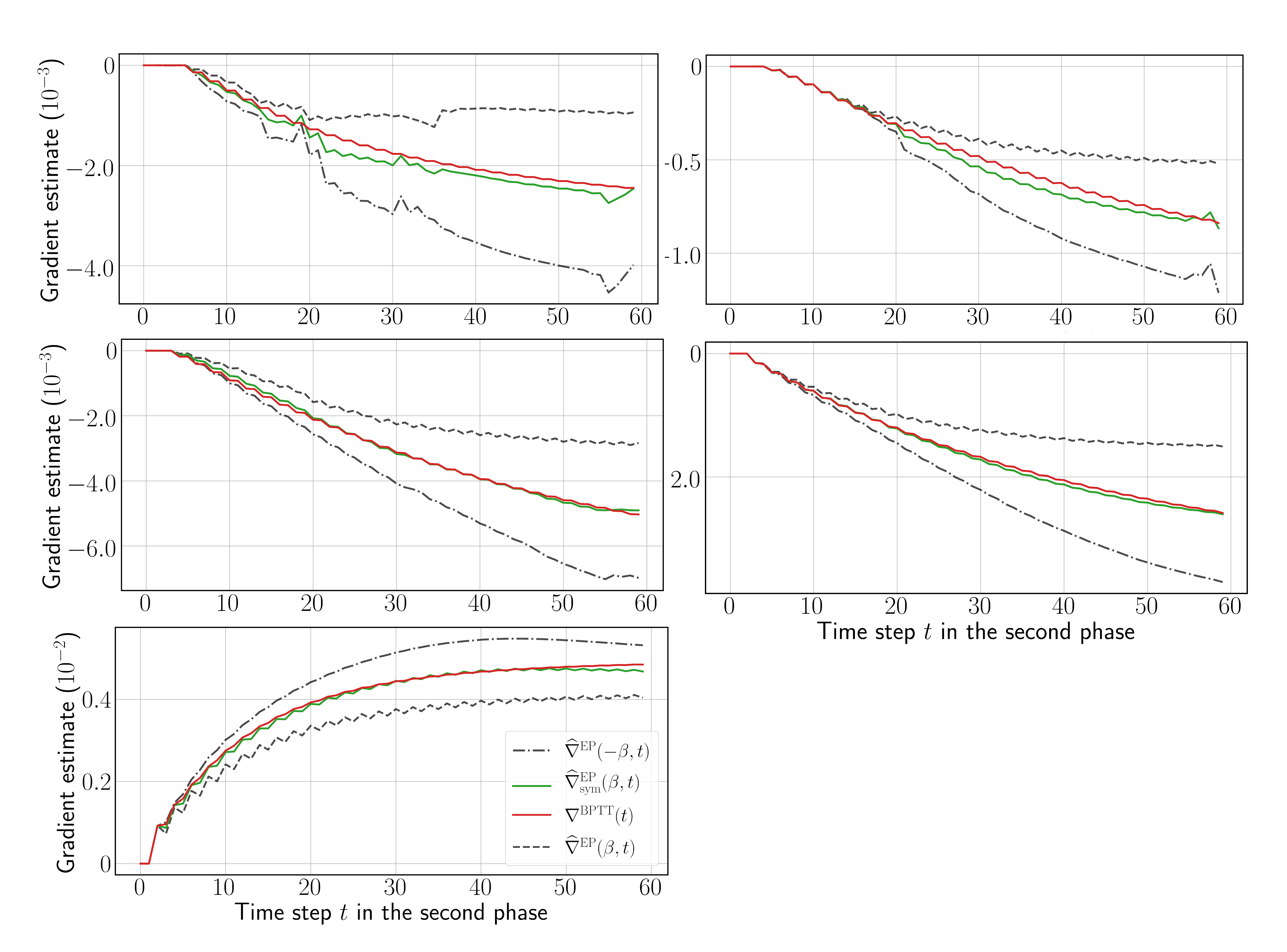}
  \caption{Layer-wise comparison between EP gradient estimates and BPTT gradients for 5 layers deep CNN on CIFAR-10 Data. Layer index increases from top to bottom, left to right, top-left being the first layer.
  }
  \label{fig:morecurve}
\end{figure}

\end{document}